\documentclass[11pt]{article}
\usepackage{fullpage}
\usepackage[numbers]{natbib}
\usepackage{algorithm}
\usepackage[noend]{algorithmic}
\usepackage{amsmath,amsthm,amsfonts,amssymb}
\usepackage{amsmath}
\usepackage{hyperref}
\usepackage{color}
\usepackage{xcolor}
\usepackage{mathrsfs}
\usepackage{enumitem}
\usepackage{bm}
\usepackage{multirow}
\usepackage{booktabs}
\usepackage{makecell}
\usepackage{graphicx}
\usepackage{subfigure}
\usepackage{caption}
\usepackage{thmtools}
\usepackage{thm-restate}
\usepackage{setspace}
\usepackage{makecell}
\definecolor{light-gray}{gray}{0.85}
\usepackage{colortbl}
\usepackage{hhline}
\usepackage{accents}

\newcommand{\defeq}{\mathrel{\mathop:}=}

\newcommand{\argmax}{\mathop{\rm argmax}}

\newcommand{\st}{\mathrm{s.t.~}}

\newcommand{\E}{\mathbb{E}}
\newcommand{\D}{\mathbb{D}}
\renewcommand{\P}{\mathbb{P}}

\newcommand{\la}{\langle}
\newcommand{\ra}{\rangle}

\newcommand{\cO}{\mathcal{O}}
\newcommand{\tlO}{\mathcal{\tilde{O}}}

\newcommand{\N}{\mathbb{N}}

\newcommand{\cF}{\mathcal{F}}

\newcommand{\cS}{\mathcal{S}}
\newcommand{\cA}{\mathcal{A}}
\newcommand{\cB}{\mathcal{B}}

\newenvironment{proof-sketch}{\noindent{\bf Proof Sketch}
  \hspace*{1em}}{\qed\bigskip\\}
\newenvironment{proof-idea}{\noindent{\bf Proof Idea}
  \hspace*{1em}}{\qed\bigskip\\}
\newenvironment{proof-of-lemma}[1][{}]{\noindent{\bf Proof of Lemma {#1}}
  \hspace*{1em}}{\qed\bigskip\\}
\newenvironment{proof-of-proposition}[1][{}]{\noindent{\bf
    Proof of Proposition {#1}}
  \hspace*{1em}}{\qed\bigskip\\}
\newenvironment{proof-of-theorem}[1][{}]{\noindent{\bf Proof of Theorem {#1}}
  \hspace*{1em}}{\qed\bigskip\\}
\newenvironment{inner-proof}{\noindent{\bf Proof}\hspace{1em}}{
  $\bigtriangledown$\medskip\\}
\newenvironment{proof-attempt}{\noindent{\bf Proof Attempt}
  \hspace*{1em}}{\qed\bigskip\\}

\usepackage{times}

\newtheorem{assumption}{Assumption}

\newtheorem{theorem}{Theorem}
\newtheorem{lemma}[theorem]{Lemma}
\newtheorem{corollary}[theorem]{Corollary}

\newtheorem{proposition}[theorem]{Proposition}

\theoremstyle{definition}
\newtheorem{definition}[theorem]{Definition}

\renewcommand{\st}{^{\text{st}}}

\renewcommand{\th}{^{\text{th}}}

\newcommand{\setto}{\leftarrow}

\newcommand{\low}[1]{\underline{#1}}

\newcommand{\MG}{{\rm MG}}

\newcommand{\AVGReg}{\xi}
\newcommand{\CUMReg}{\Xi}
\newcommand{\AVGSwapReg}{\xi_{\text{sw}}}
\newcommand{\CUMSwapReg}{\Xi_{\text{sw}}}
\newcommand{\AVGRegret}{\mathcal{R}}
\newcommand{\AVGSwapRegret}{\mathcal{R}_{\text{swap}}}
\newcommand{\wAVGRegret}{\tilde{\mathcal{R}}}
\newcommand{\wAVGSwapRegret}{\tilde{\mathcal{R}}_{\text{swap}}}

\title{\bf V-Learning---A Simple, Efficient, Decentralized Algorithm \\
for Multiagent RL}

\author{%
   Chi Jin \\
   Princeton University \\
   \texttt{chij@princeton.edu} \\
   \and
   Qinghua Liu \\
   Princeton University \\
   \texttt{qinghual@princeton.edu} \\
   \and
   Yuanhao Wang \\
   Princeton University \\
   \texttt{yuanhao@princeton.edu} \\
   \and
   Tiancheng Yu \\
   MIT \\
   \texttt{yutc@mit.edu} \\
}

\begin{document}
\maketitle


\begin{abstract}

A major challenge of multiagent reinforcement learning (MARL) is \emph{the curse of multiagents}, where the size  of the joint action space scales exponentially with the number of agents. 
This remains to be a bottleneck for designing efficient MARL algorithms even in a basic scenario with finitely many states and actions. This paper resolves this challenge for the model of episodic Markov games. We design a new class of fully decentralized algorithms---V-learning, which provably learns Nash equilibria (in the two-player zero-sum setting), correlated equilibria and coarse correlated equilibria (in the multiplayer general-sum setting) in a number of samples that only scales with $\max_{i\in[m]} A_i$, where $A_i$ is the number of actions for the $i\th$ player. This is in sharp contrast to the size of the joint action space which is $\prod_{i=1}^m A_i$.

V-learning (in its basic form) is a new class of single-agent RL algorithms that convert any adversarial bandit algorithm with suitable regret guarantees into a RL algorithm. Similar to the classical Q-learning algorithm, it performs incremental updates to the value functions. Different from Q-learning, it only maintains the estimates of V-values instead of Q-values. This key difference allows V-learning to achieve the claimed guarantees in the MARL setting by simply letting all agents run V-learning independently.

\end{abstract}

\section{Introduction}
\label{sec:introduction}

A wide range of modern artificial intelligence challenges can be cast as multi-agent reinforcement learning (MARL) problems, in which agents learn to make a sequence of decisions 
in the presence of other agents whose decisions will influence the outcome and can adapt to the strategies of the agents. 
Modern MARL systems have achieved significant success recently on a rich set of traditionally challenging tasks, including the game of GO \citep{silver2016mastering,silver2017mastering}, Poker \citep{brown2019superhuman}, real-time strategy games \citep{vinyals2019grandmaster, openaidota}, decentralized controls or multiagent robotics systems \citep{brambilla2013swarm}, autonomous driving \citep{shalev2016safe}, as well as complex social scenarios such as hide-and-seek~\citep{baker2020emergent}. While  single-agent RL has been the focus of recent intense theoretical study, MARL has been comparatively underexplored, which leaves several fundamental questions open even in the basic model of \emph{Markov games} \cite{shapley1953stochastic} with finitely many states and actions.

One such unique challenge of  MARL is \emph{the curse of multiagents}---let $A_i$ be the number of actions for the $i\th$ player, then the number of possible joint actions (as well as the number of parameters to specify a Markov game) scales with $\prod_{i=1}^m A_i$, which grows exponentially with the number of agents $m$.
This remains to be a bottleneck even for the best existing algorithms for learning Markov games. In fact, a majority of these algorithms adapt the classical single-agent algorithms, such as value iteration or Q-learning, into the multiagent setting \cite{NEURIPS2020_172ef5a9,liu2021sharp}, whose sample complexity scales at least linearly with respect to $\prod_{i=1}^m A_i$. This is prohibitively large in practice even for fairly small multiagent applications, say only ten agents are involved with ten actions available for each agent.

Another challenge of the MARL is to design \emph{decentralized} algorithms. While a centralized algorithm requires the existence of a centralized controller which gathers all information and jointly optimizes the policies of all agents, a decentralized algorithm allows each agent to only observe her own actions and rewards while optimizing her own policy independently. Decentralized algorithms are often preferred over centralized algorithms in practice since (1) decentralized algorithms are typically cleaner, easier to implement; (2) decentralized algorithms are more versatile as the individual learners are indifferent to the interaction and the number of other agents; and (3) they are also faster due to less communication required. 
While several provable decentralized MARL algorithms have been developed \citep[see, e.g.,][]{zhang2018fully,sayin2021decentralized,daskalakis2021independent}, they either have only asymptotic guarantees or work only under certain reachability assumptions (see Section \ref{sec:related}). The existing provably \emph{efficient} algorithms for general Markov games (without further assumptions) are exclusively centralized algorithms \citep{bai2020provable,xie2020learning,liu2021sharp}.

This motivates us to ask the following open question:
\begin{center}
\textbf{Can we design \emph{decentralized} MARL algorithms that \emph{break the curse of multiagents?}}
\end{center}

\begin{table*}[!t]
    \renewcommand{\arraystretch}{1.6} 
    \centering
    \caption{\label{table:summary_results} A summary of sample complexities for V-learning under different settings. Here $H$ is the length of each episode, $S$ is the number of states, $A = \max_{i\in[m]} A_i$ is the largest number of actions for each agent, $\epsilon$ is the error tolerance. An information theoretical lower bound for all objectives is $\Omega(H^3 S A /\epsilon^2)$ \cite{jin2018q,bai2020provable}. }
    \begin{tabular}{|c|c|c|}
      \hline
      \multirow{2}{*}{\textbf{Objective}} & \multicolumn{2}{c|}{\textbf{Multiplayer General-sum}}  \\ \cline{2-3}
      & \textbf{Two-player Zero-sum} & - \\ \hline\hline
      Nash Equilibria & \cellcolor{light-gray} $\tlO(H^5 S A /\epsilon^2)$  & PPAD-complete\\\hline
      Coarse Correlated Equilibria  &  \multicolumn{2}{c|}{ \cellcolor{light-gray} $\tlO(H^5 S A /\epsilon^2)$}  \\ \hline
      Correlated Equilibria &  \multicolumn{2}{c|}{ \cellcolor{light-gray}$\tlO(H^5 S A^2 /\epsilon^2)$} \\\hline
    \end{tabular}
\end{table*}

This paper addresses both challenges mentioned above, and provide the first positive answer to this question in the basic setting of tabular episodic Markov games. We propose a new class of single-agent RL algorithms---V-learning, which converts any adversarial bandit algorithm with suitable regret guarantees into a RL algorithm. Similar to the classical Q-learning algorithm, V-learning also performs incremental updates to the values. Different from Q-learning, V-learning only maintains the V-value functions instead of the Q-value functions. We remark that the number of parameters of Q-value functions in MARL is $\cO(S\prod_{i=1}^m A_i)$, where $S$ is the number of states, while the number of parameters of V-value functions is only $\cO(S)$. This key difference allows V-learning to be readily extended to the MARL setting by simply letting all agents run V-learning independently, which gives a fully \emph{decentralized} algorithm.

We consider the standard learning objectives in the game theory---Nash equilibrium (NE), correlated equilibrium (CE) and coarse correlated equilibrium (CCE). Except for the task of finding a NE in multiplayer general-sum games which is PPAD-complete even for matrix games, we prove that V-learning finds a NE for two-player zero-sum Markov games within $\tlO(H^5 S A /\epsilon^2)$ episodes, where $H$ is the length of each episode, $A = \max_{i\in[m]} A_i$ is the maximum number of actions of all agents, and $\epsilon$ is the error tolerance for the objective. We further prove that for multiplayer general-sum Markov games, V-learning coupled with suitable adversarial bandit algorithms is capable of finding a CCE within $\tlO(H^5 S A /\epsilon^2)$ episodes, and a CE within $\tlO(H^5 S A^2 /\epsilon^2)$ episodes. All sample complexities mentioned above do not grow with the number of the agents, which thus \emph{break the curse of multiagents} (See Table \ref{table:summary_results} for a summary).

\subsection{Related work}
\label{sec:related}

In this section, we focus our attention on theoretical results for the tabular setting, where the numbers of states and actions are finite. We acknowledge that there has been much recent work in RL for continuous state spaces \cite[see, e.g.,][]{jiang2017contextual,jin2020provably,zanette2020learning,jin2021bellman,xie2020learning,jin2021power}, but this setting is beyond our scope. 

\paragraph{Markov games.} Markov Game (MG), also known as stochastic game \citep{shapley1953stochastic}, is a popular model in multi-agent RL \citep{littman1994markov}. Early works have mainly focused on finding Nash equilibria of MGs under strong assumptions, such as known transition and reward \citep{littman2001friend,hu2003nash,hansen2013strategy,wei2020linear},
or certain reachability conditions \citep{wei2017online, wei2021last} (e.g., having access to simulators \citep{jia2019feature,sidford2020solving,zhang2020model}) that alleviate the challenge in exploration.

A recent line of works provide non-asymptotic guarantees for learning two-player zero-sum tabular MGs without
further structural assumptions.
\citet{bai2020provable} and \citet{xie2020learning} develop the first provably-efficient learning algorithms in MGs based on optimistic value iteration. \citet{liu2021sharp} improves upon these works and achieve best-known sample complexity for finding an $\epsilon$-Nash equilibrium---$\cO(H^3 SA_1A_2/\epsilon^2)$ episodes.

For multiplayer general-sum tabular MGs, \citet{liu2021sharp} is the only existing work that provides non-asymptotic guarantees in the exploration setting. It proposes centralized model-based algorithms based on value-iteration, and shows that Nash equilibria (although computationally inefficient), CCE and CE can be all learned within $\cO(H^4 S^2 \prod_{j=1}^m A_j/\epsilon^2)$ episodes. Note this result suffers from the curse of multiagents.

V-learning---initially coupled with the FTRL algorithm as adversarial bandit subroutine---is firstly proposed in the conference version of this paper \cite{NEURIPS2020_172ef5a9}, for finding Nash equilibria in the two-player zero-sum setting.
During the preparation of this draft, we note two very recent independent works \cite{song2021can,mao2021provably}, whose results partially overlap with the results of this paper in the multiplayer general-sum setting. In particular, \citet{mao2021provably} use V-learning with stablized online mirror descent as adversarial bandit subroutine, and learn $\epsilon$-CCE in $\cO(H^6 SA/\epsilon^2)$ episodes, where $A = \max_{j\in[m]}A_j$. This is one $H$ factor larger than what is required in Theorem \ref{thm:CCE_main_result} of this paper. \citet{song2021can} considers similar V-learning style algorithms for learning both $\epsilon$-CCE and $\epsilon$-CE. For the latter objective, they require $\cO(H^6 SA^2/\epsilon^2)$ episodes which is again one $H$ factor larger than what is required in Theorem \ref{thm:CE_main_result} of this paper. \citet{song2021can} also considers Markov potential games, which is beyond the scope of this paper. We remark that both parallel works have not presented V-learning as a generic class of algorithms which can be coupled with any adversarial bandit algorithms with suitable regret guarantees in a black-box fashion.

\paragraph{Strategic games.} Strategic game is one of the most basic game forms studied in the game theory literature \cite{osborne1994course}. It can be viewed as Markov games \emph{without} state and transition. The fully decentralized algorithm that breaks the curse of multiagents is known in the setting of strategic games. By independently running no-regret (or no-swap-regret) algorithm for all agents, one can find Nash Equilibria (in the two-player zero-sum setting), correlated equilibria and coarse correlated equilibria (in the multiplayer general-sum setting) in a number of samples that only scales with $\max_{i\in[m]} A_i$ \cite{cesa2006prediction,hart2000simple,blum2007external}. However, such successes do not directly extend to the Markov games due to the additional temporal structures involving both states and transition. In particular, there is no computationally efficient no-regret algorithm for Markov games \cite{radanovic2019learning, NEURIPS2020_172ef5a9}.

\paragraph{Extensive-form games.}
There is another long line of research on MARL based on the model of extensive-form games (EFG) \cite[see, e.g.,][]{koller1992complexity,gilpin2006finding,zinkevich2007regret,brown2018superhuman,brown2019superhuman,celli2020no}. EFGs can be viewed as special cases of Markov games where any state at the $h\th$ step can be  reached from only one state at the $(h-1)\th$ step (due to tree structure of the game). Therefore, results on learning EFGs do not directly imply results for learning MGs.

\paragraph{Decentralized MARL} There is a long line of \emph{empirical} works on decentralized MARL \citep[see, e.g.,][]{lowe2017multi,iqbal2019actor,sunehag2017value,rashid2018qmix,son2019qtran}. A majority of these works focus on the cooperative setting. They additionally attack the challenge where each agent can only observe a part of the underlying state, which is beyond the scope of this paper. For theoretical results, \citet{zhang2018fully} consider the cooperative setting while \citet{sayin2021decentralized} study the two-player zero-sum Markov games. Both develop decentralized MARL algorithms but provide only asymptotic guarantees. \citet{daskalakis2021independent} analyze the convergence rate of independent policy gradient method in episodic two-player zero-sum MGs. Their result requires the additional reachability assumptions (concentrability) which alleviates the difficulty of exploration.

\paragraph{Single-agent RL}
There is a rich literature on reinforcement learning in MDPs \citep[see e.g.][]{jaksch2010near, osband2014generalization, azar2017minimax, dann2017unifying, strehl2006pac, jin2018q,jin2021bellman}. MDPs are special cases of Markov games, where 
only a single agent interacts with a stochastic environment. For the tabular episodic setting with nonstationary dynamics and no simulators, the best sample complexity achieved by existing model-based and model-free algorithms are $\tilde{\mathcal{O}}(H^3SA/\epsilon^2)$ (achieved by value iteration \cite{azar2017minimax}) and $\tilde{\mathcal{O}}(H^4SA/\epsilon^2)$ (achieved by Q-learning \cite{jin2018q}), respectively, where $S$ is the number of states, $A$ is the number of actions, $H$ is the length of each episode. Both of them (nearly) match the lower bound $\Omega(H^3SA/\epsilon^2)$~\cite{jaksch2010near, osband2016lower, jin2018q}.

\section{Preliminaries} \label{sec:prelim}

We consider the model of Markov Games (MG) \cite{shapley1953stochastic} (also known as stochastic games in the literature) in its most generic---multiplayer general-sum form. Formally, we denote an tabular episodic MG with $m$ players by a tuple $\MG(H, \cS, \{\cA_i\}_{i=1}^m, \P, \{r_i\}_{i=1}^m)$, where $H$, $\cS$ denote the length of each episode and the state space with $|\cS| = S$.
 $\cA_i$ denotes the action space for the $i\th$ player and $|\cA_i| = A_i$. We let $\bm{a}:=(a_{1},\cdots,a_{m})$ denote the (tuple of) joint actions by all $m$ players, and $\cA = \cA_1 \times \ldots \times \cA_m$. $\P = \{\P_h\}_{h\in[H]}$ is a collection of transition matrices, so that $\P_h ( \cdot | s, \bm{a}) $ gives the distribution of the next state if actions $\bm{a}$ are taken at state $s$ at step $h$, and $r_i = \{r_{i,h}\}_{h\in[H]}$ is a collection of reward functions for the $i^{\text{th}}$ player, so that $r_{i,h}(s, \bm{a}) \in[0, 1]$ gives the deterministic reward received by the $i^{\text{th}}$ player if actions $\bm{a}$ are taken at state $s$ at step $h$. 
 \footnote{Our results directly generalize to random reward functions, since learning transitions is more difficult than learning rewards.} We remark that since the relation among the rewards of different agents can be arbitrary, this model of MGs incorporates both cooperation and competition.

In each episode, we start with a \emph{fixed initial state} $s_1$.
\footnote{While we assume a fixed initial state for notational simplicity, our results readily extend to the setting where the initial state is sampled from a fixed initial distribution.} 
At each step $h \in [H]$, each player $i$ observes state $s_h \in \cS$, picks action $a_{i,h}\in\cA_i$ simultaneously, observes the actions played by other players,\footnote{We assume the knowledge of other players' actions for the convenience of later definitions. We remark that the V-learning algorithm introduced in this paper does \emph{not} require knowing the actions of other players.} and receives her own reward $r_{i,h}(s_h, \bm{a}_h)$. Then the environment transitions to the next state
$s_{h+1}\sim\P_h(\cdot | s_h, \bm{a}_h)$. The episode ends when $s_{H+1}$ is reached.


\paragraph{Policy, value function}
A (\emph{random}) \emph{policy} $\pi_i$ of the $i^{\rm th}$ player is a set of $H$ maps $\pi_i \defeq \big\{ \pi_{i, h}: \Omega \times (\cS \times \cA)^{h-1}\times \cS \rightarrow \cA_i \big\}_{h\in [H]}$, where $\pi_{i, h}$ maps a random sample $\omega$ from a probability space $\Omega$ and a history of length $h$---say $\tau_h := (s_1, \bm{a}_1, \cdots, s_h)$, to an action in $\cA_i$. To execute policy $\pi_i$, we first draw a random sample $\omega$ at the beginning of the episode. Then, at each step $h$, the $i\th$ player simply takes action $\pi_{i,h}(\omega, \tau_h)$. We note here $\omega$ is shared among all steps $h \in[H]$. $\omega$ encodes both the correlation among steps and the individual randomness of each step. We further say a policy $\pi_i$ is \emph{deterministic} if $\pi_{i,h}(\omega, \tau_h) = \pi_{i,h}(\tau_h)$ which is independent of the choice of $\omega$. 

An important subclass of policy is \emph{Markov policy}, which can be defined as $\pi_i \defeq \big\{ \pi_{i, h}: \Omega \times \cS \rightarrow \cA_i \big\}_{h\in [H]}$. Instead of depending on the entire history, a Markov policy takes actions only based on the current state. Furthermore, the randomness in  each step of Markov policy is independent. Therefore, when it is clear from the context, we write Markov policy as $\pi_i \defeq \big\{ \pi_{i, h}: \cS \rightarrow \Delta_{\cA_i} \big\}_{h\in [H]}$, where $ \Delta_{\cA_i}$ denotes the simplex over $\cA_i$. We also use notation $\pi_{i,h}(a|s)$ to denote the probability of the $i\th$ agent taking action $a$  at state $s$ at step $h$.


A joint (potentially correlated) policy is a set of policies $\{\pi_i\}_{i=1}^m$, where the same random sample $\omega$ is shared among all agents, which we denote as $\pi = \pi_1 \odot \pi_2 \odot \ldots \odot \pi_m$. We also denote $\pi_{-i} = \pi_1 \odot \ldots \pi_{i-1} \odot \pi_{i+1} \odot \ldots \odot \pi_m$ to be the joint policy excluding the $i\th$ player. A special case of joint policy is the \emph{product policy} where the random sample has special form $\omega = (\omega_1, \ldots, \omega_m)$, and for any $i\in [m]$, $\pi_i$ only uses the randomness in $\omega_i$, which is independent of remaining $\{\omega_j\}_{j\neq i}$, which we denote as $\pi = \pi_1 \times \pi_2 \times \ldots \times \pi_m$.


We define the value function $V^\pi_{i, 1}(s_1)$ as the expected cumulative reward that the $i\th$ player will receive if the game starts at initial state $s_1$ at the $1\st$ step and all players follow joint policy $\pi$:
\begin{equation} \label{eq:value_general_policy}
\textstyle V^{\pi}_{i, 1}(s_1) \defeq \E_{\pi}\left[\left.\sum_{h =
        1}^H r_{i,h}(s_{h}, \bm{a}_{h}) \right| s_1 \right].
\end{equation}
where the expectation is taken over the randomness in transition and the random sample $\omega$ in policy $\pi$.


\paragraph{Best response and strategy modification}
For any strategy $\pi_{-i}$, the \emph{best response} of the $i\th$ player is defined as a policy of the $i\th$ player which is independent of the randomness in $\pi_{-i}$, and achieves the highest value for herself conditioned on all other players deploying $\pi_{-i}$. In symbol, the best response is the maximizer of $\max_{\pi'_i} V_{i, 1}^{\pi'_i \times \pi_{-i}}(s_1)$ whose value we also denote as $V_{i, 1}^{\dagger, \pi_{-i}}(s_1)$ for simplicity. By its definition, we know the best response can always be achieved at \emph{deterministic} policies.

A \emph{strategy modification} $\phi_i$ for the $i\th$ player is a set of maps $\phi_i:=\{\phi_{i,h}: (\cS \times \cA)^{h-1}\times \cS \times \cA_i \rightarrow \cA_i\}$,\footnote{Here, we only introduce the deterministic strategy modification for simplicity of notation, which is sufficient for discussion in the context of this paper. The random strategy modification can also be defined by introducing randomness in $\phi_i$ which is independent of randomness in $\pi_i$ and $\pi_{-i}$. It can be shown that the best strategy modification can always be deterministic.} where $\phi_{i, h}$ can depend on the history $\tau_h$ and maps actions in $\cA_i$ to different actions in $\cA_i$. For any policy of the $i\th$ player $\pi_i$, the modified policy (denoted as $\phi_i \diamond \pi_i$) changes the action $\pi_{i,h}(\omega, \tau_h)$ under random sample $\omega$ and history $\tau_h$ to $\phi_i(\tau_h, \pi_{i,h}(\omega, \tau_h))$. For any joint policy $\pi$, we define the best strategy modification of the $i\th$ player as the maximizer of $\max_{\phi_i} V_{i, 1}^{(\phi_i \diamond \pi_i) \odot \pi_{-i}}(s_1)$. 

Different from the best response, which is completely independent of the randomness in $\pi_{-i}$, the best strategy modification changes the policy of the $i\th$ player while still utilizing the shared randomness among $\pi_i$ and $\pi_{-i}$. Therefore, the best strategy modification is more powerful than the best response: formally one can show that 
$\max_{\phi_i} V_{i, 1}^{(\phi_i \diamond \pi_i) \odot \pi_{-i}}(s_1) \ge \max_{\pi'_i} V_{i, 1}^{\pi'_i \times \pi_{-i}}(s_1)$ for any policy $\pi$.

\subsection{Learning objectives}

A special case of Markov game is Markov Decision Process (MDP). One can show there always exists an optimal policy $\pi^\star = \argmax_{\pi} V^\pi_1(s_1)$. Denote the value of the optimal policy as $V^\star$. The objective of learning MDPs is to find an $\epsilon$-optimal policy $\pi$, which satisfies $V^\star_1(s_1) - V^\pi_1(s_1) \le \epsilon$.


For Markov games, there are three common learning objectives in the game theory literature---Nash Equilibrium, Correlated Equilibrium (CE) and Coarse Correlated Equilibrium (CCE).

First, a Nash equilibrium is defined as a product policy where no player can increase her value by changing only her own policy. Formally,
\begin{definition}[Nash Equilibrium] \label{def:NE}
A \emph{product} policy $\pi$ is a \textbf{ Nash equilibrium} if $\max_{i\in[m]}{( V_{i,1}^{\dag, \pi_{-i}}-V_{i,1}^{\pi } )}( s_{1} ) = 0$. A \emph{product} policy $\pi$ is an $\epsilon$-approximate Nash equilibrium if $\max_{i\in[m]}{( V_{i,1}^{\dag, \pi_{-i}}-V_{i,1}^{\pi } )}( s_{1} ) \le \epsilon$.
\end{definition}

We remark that, except for the special case of two-player zero-sum Markov games where reward $r_{2, h} = -r_{1,h}$\footnote{Technically, to ensure $r_{2,h} \in [0, 1]$, we choose $r_{2, h} = 1-r_{1,h}$. We note that adding a constant to the reward function has no effect on the equilibria, which is our learning objective.}
for any $h \in [H]$, the Nash equilibrium in general has been proved PPAD-hard to compute \cite{daskalakis2013complexity}. Therefore, we only present results for finding Nash equilibria in  two-player zero-sum MGs in this paper.

Second, a coarse correlated equilibrium is defined as a joint (potentially correlated) policy where no player can increase her value by playing a different independent strategy. In symbol,
\begin{definition}[Coarse Correlated Equilibrium] \label{def:CCE}
A joint policy $\pi$ is a \textbf{CCE} if $\max_{i\in[m]}{( V_{i,1}^{\dag, \pi_{-i}}-V_{i,1}^{\pi } )}( s_{1} ) = 0$. A joint policy $\pi$ is a $\epsilon$-approximate CCE if $\max_{i\in[m]}{( V_{i,1}^{\dag, \pi_{-i}}-V_{i,1}^{\pi } )}( s_{1} ) \le \epsilon$.
\end{definition}
The only difference between Definition \ref{def:NE} and Definition \ref{def:CCE} is that Nash equilibrium requires the policy $\pi$ to be a product policy while CCE does not. Thus, it is clear that CCE is a relaxed notion of Nash equilibrium, and a Nash equilibrium is always a CCE.

Finally, a correlated equilibrium is defined as a joint (potentially correlated) policy where no player can increase her value by using a strategy modification. In symbol,
\begin{definition}[Correlated Equilibrium] \label{def:CE}
A joint policy $\pi$ is a \textbf{CE} if $\max_{i\in[m]} \max_{\phi_i} {( V_{i, 1}^{(\phi_i \diamond \pi_i) \odot \pi_{-i}}-V_{i,1}^{\pi } )}( s_{1} ) = 0$. A joint policy $\pi$ is a $\epsilon$-approximate CE if $\max_{i\in[m]} \max_{\phi_i} {( V_{i, 1}^{(\phi_i \diamond \pi_i) \odot \pi_{-i}}-V_{i,1}^{\pi } )}( s_{1} ) \le \epsilon$.
\end{definition}
In Markov games, we also have that a Nash equilibrium is a CE, and a CE is a CCE (see Proposition \ref{prop:nash-ce-cce} in Appendix \ref{app:notation} for more details).

\section{V-Learning Algorithm}
\label{sec:single-agent}

In this section, we introduce V-learning algorithm as a new class of single-agent RL algorithms, which converts any adversarial
bandit algorithm with suitable regret guarantees into a RL algorithm. We also present its theoretical guarantees for finding a nearly optimal policy in the single-agent setting.

\subsection{Training algorithm}

\begin{algorithm}[t]
   \caption{\textsc{V-learning}}
   \label{algorithm:V}
   \begin{algorithmic}[1]
      \STATE {\bfseries Initialize:} for any $(s, a, h)$,
      $V_{h}(s)\setto H+1-h$, $N_{h}(s)\setto 0$, $\pi_h(a|s) \setto 1/A$.
      \FOR{episode $k=1,\dots,K$}
      \STATE receive $s_1$.
      \FOR{step $h=1,\dots, H$}
      \STATE take action $a_h \sim \pi_h(\cdot |s_h)$, observe reward $r_h$ and next state $s_{h+1}$.
      \label{line:action}
      \STATE $t=N_{h}(s_h)\setto N_{h}(s_h) + 1$.
   \label{line:update_counter}
      \STATE $\tilde{V}_h(s_h) \setto (1-\alpha_t)\tilde{V}_h(s_h)+ \alpha_t(r_h+V_{h+1}(s_{h+1})+\beta_t)$.
      \label{line:update_V}
            \STATE $V_h(s_h) \setto \min\{H+1-h, \tilde{V}_h(s_h)\}$.
      \label{line:monotone}
      \STATE \label{line:bandits} $\pi_h(\cdot | s_h) \leftarrow \textsc{Adv\_Bandit\_Update}(a_h, \frac{H -r_h-V_{h+1}(s_{h+1})}{H})$ on $(s_h, h)\th$ adversarial bandit.
      \ENDFOR
      \ENDFOR
   \end{algorithmic}
\end{algorithm}

To begin with, we describe the V-learning algorithm (Algorithm \ref{algorithm:V}). It maintains a value $V_h(s)$, a counter $N_h(s)$, and a policy $\pi_h(\cdot|s)$ for each state $s$ and step $h$, and initializes them to be the max value, $0$, and uniform distribution respectively. V-learning also instantiates $S\times H$ different adversarial bandit algorithms---one for each $(s, h)$ pair. At each step $h$ in each episode $k$, the algorithm performs three major steps:
\begin{itemize}
	\item Policy execution (Line \ref{line:action}-\ref{line:update_counter}):  the algorithm takes action $a_h$ according to the maintained $\pi_h$, then observes the reward $r_h$ and the next state $s_{h+1}$, and increases the counter $N_h(s_h)$ by $1$.
	\item $V$-value update (Line \ref{line:update_V}-\ref{line:monotone}): the algorithm performs incremental update to the value function:
   \begin{equation} \label{eq:incremental_V}
   \tilde{V}_h(s_h) \setto (1-\alpha_t)\tilde{V}_h(s_h)+ \alpha_t(r_h+V_{h+1}(s_{h+1})+\beta_t)
   \end{equation}
   here $\alpha_t$ is the learning rate, and $\beta_t$ is the bonus to promote optimism (and exploration). The choices of both quantities will be specified later. Next, we simply update $V_h$ as a truncated version of $\tilde{V}_h$.
	\item Policy update (Line \ref{line:bandits}): the algorithm feeds the action $a_h$ and its ``loss'' $\frac{H-r_h+V_{h+1}(s_{h+1})}{H}$ to the $(s_h, h)\th$ adversarial bandit algorithm, and receives the updated policy $\pi_h(\cdot | s_h)$.
\end{itemize}

Throughout this paper, we will always use the following learning rate $\alpha_t$. We also define an auxiliary sequence $\{\alpha_t^i\}_{i=1}^t$ based on the learning rate, which will be frequently used across the paper.
\begin{equation} \label{eq:hyper_V}
  \alpha_t=\frac{H+1}{H+t}, \quad \alpha_t^0 = \prod_{j=1}^{t}(1-\alpha_j), \quad \alpha_t^i = \alpha_i \prod_{j=i+1}^t (1-\alpha_j).
\end{equation}
We remark that our incremental update \eqref{eq:incremental_V} bears significant similarity to Q-learning, and our choice of learning rate is precisely the same as the choice in Q-learning \citep{jin2018q}. However, a key difference is that the V-learning algorithm maintains V-value functions instead of Q-value functions. This is crucial when extending V-learning to the multiplayer setting where the number of parameters of Q-value functions becomes $\cO(HS \prod_{i=1}^m A_i)$ while the number of parameters of V-value functions is only $\cO(HS)$. Since V-learning does not use action-value functions, it resorts to adversarial bandit algorithms to update its policy.

\makeatletter
\renewcommand{\ALG@name}{Protocol}
\makeatother

\begin{algorithm}[t]

\caption{\textsc{Adversarial Bandit Algorithm}}
   \label{pro:adversarial_bandit}
   \begin{algorithmic}[1]
      \STATE {\bfseries Initialize:} for any $b$, $\theta_1(b) \setto 1/B.$
      \FOR{step $t=1,\dots,T$}
      \STATE adversary chooses loss $\ell_t$.
      \STATE take action $b_t \sim \theta_t$, observe noisy bandit-feedback $\tilde{\ell}_t(b_t)$.
      \STATE $\theta_{t+1} \leftarrow \textsc{Adv\_Bandit\_Update}(b_t, \tilde{\ell}_t(b_t))$. 
      \ENDFOR
   \end{algorithmic}
\end{algorithm}

\paragraph{\textsc{Adv\_Bandit\_Update} subroutine:} Consider a multi-arm bandit problem with adversarial loss, where we denote the action set by $\cB$ with $|\cB| = B$. At round $t$, the learner picks a strategy (distribution over actions) $\theta_t \in \Delta_\cB$, and the adversary chooses a loss vector $\ell_t \in [0, 1]^B$. Then the learner takes an action $b_t$ that is sampled from distribution $\theta_t$, and receives a noisy bandit-feedback $\tilde{\ell}_t(b_t) \in [0, 1]$ where $\E[\tilde{\ell}_t(b_t)|\ell_t, b_t] = \ell_t(b_t)$. Then, the adversarial bandit algorithm performs updates based on $b_t$ and $\tilde{\ell}_t(b_t)$, and outputs the strategy for next round $\theta_{t+1}$, which we abstract as $\theta_{t+1} \leftarrow \textsc{Adv\_Bandit\_Update}(b_t, \tilde{\ell}_t(b_t))$ (see Protocol \ref{pro:adversarial_bandit}).

\subsection{Output policy}



\makeatletter
\renewcommand{\ALG@name}{Algorithm}
\makeatother

\begin{algorithm}[t]
   \caption{\textsc{Executing Output Policy $\hat{\pi}$ of V-learning}}
   \label{algorithm:V-policy}
   \begin{algorithmic}[1]
   \STATE sample $k \leftarrow \text{Uniform}([K])$.
      \FOR{step $h=1, \dots,H$}
      \STATE observe $s_{h}$, and set $t \setto N^{k}_h(s_{h})$.
      \STATE set $k \setto k^i_{h}(s_{h})$, where $i\in [t]$ is sampled with probability $\alpha^i_t$. \label{line:sample_output}
      \STATE take action $a_{h} \sim \pi^{k}_h(\cdot|s_h)$.
      \label{line:take_action}
      \ENDFOR
   \end{algorithmic}
\end{algorithm}


We define the final output policy $\hat{\pi}$ of V-learning by how to execute this policy (see Algorithm \ref{algorithm:V-policy}). Let $V^k, N^k, \pi^k$ be the value, counter and policy maintained by V-learning algorithm at \emph{the beginning} of  episode $k$. The output policy maintains a scalar $k$, which is initially uniformly sampled from $[K]$. At each step $h$, after observing $s_h$, $\hat{\pi}$ plays a mixture of policy $\{\pi_h^{k^i}(\cdot|s_h)\}_{i=1}^t$ with corresponding probability $\{\alpha_t^i\}_{i=1}^t$ defined in \eqref{eq:hyper_V}. Here $t=N_h^{k}(s_h)$ is the number of times $s_h$ is visited at step $h$ at the beginning of episode $k$, and $k^i$ is short for $k^i_h(s_h)$ which is the index of the episode when $s_h$ is visited at step $h$ for the $i\th$ time. After that, $\hat{\pi}$ sets $k$ to be the index $k^i_h(s_h)$ whose policy is just played within the mixture, and continue the same process for the next step. This mixture form of output policy $\hat{\pi}$ is mainly due to the incremental updates of V-learning. One can show that, if omitting the optimistic bonus, $V^K_1(s_1)$ computed in the V-learning algorithm is a stochastic estimate of the value of policy $\hat{\pi}$.

We remark that $\hat{\pi}$ is not a Markov policy, but a general random policy (see Definition in Section \ref{sec:prelim}), which can be written as a set of maps $\{\pi_h: \Omega \times \cS^h \rightarrow \cA_i\}$. The choice of action at each step $h$ depends on a joint randomness $\omega \in \Omega$ which is shared among all steps, and the history of past states $(s_1, \ldots, s_h)$. In Section \ref{sec:monotone}, we will further introduce a simple monotone technique that allows V-learning to output a Markov policy in both the single-agent and the two-player zero-sum setting.




\subsection{Single-agent guarantees}

We first state our requirement for the adversarial bandit algorithm used in V-learning, which is to have a high probability \emph{weighted} external regret guarantee as follows. The weights $\{\alpha_t^i\}_{i=1}^t$ are defined in \eqref{eq:hyper_V}.

\begin{assumption} \label{ass:external_regret}
For any $t \in \N$ and any $\delta \in (0, 1)$, with probability at least $1-\delta$, we have
\begin{equation}\label{eq:external_reg}
\max_{\theta \in \Delta_\cB}\sum_{i=1}^t \alpha_t^{i} [\la \theta_i, \ell_i \ra - \la \theta, \ell_i\ra ] \le \AVGReg(B, t, \log(1/\delta)).
\end{equation}
We further assume the existence of an upper bound $\CUMReg(B, t, \log(1/\delta)) \ge \sum_{t'=1}^t \AVGReg(B, t', \log(1/\delta))$where (i) $\AVGReg(B, t, \log(1/\delta))$ is non-decreasing in $B$ for any $t, \delta$; (ii) $\CUMReg(B, t, \log(1/\delta))$ is concave in $t$ for any $B$, $\delta$.
\end{assumption}



Assumption \ref{ass:external_regret} can be satisfied by modifying many existing algorithms with unweighted external regret to the weighted setting. In particular, we prove that the Follow-the-Regularized-Leader (FTRL) algorithm (Algorithm \ref{algorithm:FTRL}) satisfies the Assumption \ref{ass:external_regret} with bounds $\AVGReg(B, t,\log(1/\delta))\le \cO(\sqrt{HB\log(B/\delta) /t})$ and $\CUMReg(B, t,\log(1/\delta))\le \cO(\sqrt{HBt\log(B/\delta)})$. The $H$ factor comes into the bounds because our choice of weights $\{\alpha^i_t\}$ in \eqref{eq:hyper_V} involves $H$. We refer readers to Appendix \ref{sec:bandit} for more details.


We are now ready to introduce the theoretical guarantees of V-learning for finding near-optimal policies in the single-agent setting.


\begin{theorem}
    \label{thm:single_main_result} Suppose subroutine \textsc{Adv\_Bandit\_Update} satisfies Assumption \ref{ass:external_regret}. For any $\delta \in (0, 1)$ and $K \in \N$, let $\iota = \log(HSAK/\delta)$. Choose learning rate $\alpha_t$ according to \eqref{eq:hyper_V} and bonus $\{\beta_t\}_{t=1}^K$ so that $\sum_{i=1}^t\alpha_t^i\beta_{i}=\Theta(H \AVGReg(A, t, \iota)+\sqrt{H^3\iota/t})$ for any $t \in [K]$. Then, with probability at least $1-\delta$, after running Algorithm \ref{algorithm:V} for $K$ episodes,  we have the output policy $\hat{\pi}$ by Algorithm \ref{algorithm:V-policy} satisfies
    \begin{equation*}
    V_1^\star(s_1) -  V_1^{\hat{\pi}}(s_1) \le \cO( (H^2S/K) \cdot \CUMReg(A, K/S, \iota)+\sqrt{H^5S\iota/K}).
    \end{equation*}
  In particular, when instantiating subroutine \textsc{Adv\_Bandit\_Update} by FTRL (Algorithm \ref{algorithm:FTRL}), we can choose
  $\beta_t = c \cdot \sqrt{H^3 A \iota/t}$ for some absolute constant $c$, where
  $V_1^\star(s_1) -     V_1^{\hat{\pi}}(s_1) \le \cO(\sqrt{H^5SA\iota/K} )$.
\end{theorem}




Theorem \ref{thm:single_main_result} characterizes   how fast the suboptimality of $\hat{\pi}$ decreases with respect to the total number of episode $K$. In particular, to obtain an $\epsilon$-optimal output policy $\hat{\pi}$, we only need to use a number of episodes $K=\tlO(H^5SA/\epsilon^2)$. This is $H^2$ factor larger than the information-theoretic lower bound $\Omega(H^3SA/\epsilon^2)$ in this setting \citep{jin2018q}. We remark that one extra $H$ factor is due to the incremental update and the use of learning rate in \eqref{eq:hyper_V} which is exactly the same for Q-learning algorithm \citep{jin2018q}. The other $H$ factor can be potentially improved by using refined first-order regret bound (counterparts of Bernstein concentration) in V-learning.

While V-learning seems to be no better than classical value iteration or Q-learning in the single-agent setting, its true power starts to show up in the multiagent setting: Value iteration and Q-learning require highly nontrivial efforts to adapt them to the multiagent setting, and by design they suffer from the curse of multiagents \citep{NEURIPS2020_172ef5a9,liu2021sharp}.
In the following sections, we will show that V-learning can be directly extended to the multiagent setting by simply letting all agents run V-learning independently. Furthermore, V-learning breaks the curse of multiagents.

\section{Two-player Zero-sum Markov Games}
\label{sec:2p0s}

In this section, we provide the sample efficiency guarantee for V-learning to find Nash equilibria in two-player zero-sum Markov games.  

\subsection{Finding Nash equilibria}

In the two-player zero-sum setting, we have two agents whose rewards satisfy $r_{1,h}= -r_{2, h}$ for any $h \in [H]$. Our algorithm is simply that both agents run V-learning (Algorithm \ref{algorithm:V}) independently with learning rate $\alpha_t$ as specified in \eqref{eq:hyper_V}. Each player $j$ will uses her own set of bonus $\{\beta_{j, t}\}$ that depends on the number of her actions and will be specified later. To execute the output policy, both agents simply execute Algorithm \ref{algorithm:V-policy} independently using their own intermediate policies computed by V-learning.

We have the following theorem for V-learning. For clean presentation, we denote $A=\max_{j\in[2]} A_j$.

\begin{theorem}
    \label{thm:2p0s_main_result} Suppose subroutine \textsc{Adv\_Bandit\_Update}  satisfies Assumption \ref{ass:external_regret}. For any $\delta \in (0, 1)$ and $K \in \N$, let $\iota = \log(HSAK/\delta)$. Choose learning rate $\alpha_t$ according to \eqref{eq:hyper_V} and bonus $\{\beta_{j, t}\}_{t=1}^K$ of the $j\th$ player so that $\sum_{i=1}^t\alpha_t^i\beta_{j,i}=\Theta(H\AVGReg(A_j, t, \iota)+\sqrt{H^3\iota/t})$ for any $t \in [K]$. Then, with probability at least $1-\delta$, after running Algorithm \ref{algorithm:V} for $K$ episodes, let $\hat{\pi}_1, \hat{\pi}_2$ be the output policies by Algorithm \ref{algorithm:V-policy} for each player, then we have the product policy $\hat{\pi} = \hat{\pi}_1 \times \hat{\pi}_2$ satisfies
    \begin{equation*}
    \max_{j\in[2]}[V_{j,1}^{\dag ,\hat{\pi}_{-j}}(s_1)-V_{j,1}^{\hat{\pi}}(s_1)] \le \cO( (H^2S/K) \cdot \CUMReg(A, K/S, \iota)+\sqrt{H^5S\iota/K}).
    \end{equation*}
  When instantiating \textsc{Adv\_Bandit\_Update} by FTRL (Algorithm \ref{algorithm:FTRL}), we can choose  $\beta_{j,t} = c \cdot \sqrt{H^3 A_j \iota/t}$ for some absolute constant $c$, which leads to
  $\max_{j\in[2]}[V_{j,1}^{\dag ,\hat{\pi}_{-j}}(s_1)-V_{j,1}^{\hat{\pi}}(s_1)] \le \cO(\sqrt{H^5SA\iota/K} )$.
\end{theorem}


Theorem \ref{thm:2p0s_main_result} claims that, to find an $\epsilon-$approximate Nash equilibrium, we only need to use a number of episodes $K=\tlO(H^5SA/\epsilon^2)$, where $A = \max_{j\in[2]} A_j$. In contrast, value iteration or Q-learning based algorithms require at least $\Omega(H^3 S A_1A_2/\epsilon^2)$ episodes to find Nash equilibria \cite{NEURIPS2020_172ef5a9, liu2021sharp}. Furthermore, V-learning is a fully decentralized algorithm. To our best knowledge, V-learning is the only algorithm up to today that achieves sample complexity linear in $A$ for finding Nash equilibrium in two-player zero-sum Markov games.

We remark that V-learning only performs $\cO(1)$ operations and calls subroutine $\textsc{Adv\_Bandit\_Update}$ once every time a new sample is observed. As long as the adversarial bandit algorithm used in V-learning is computationally efficient (which is the case for FTRL), V-learning itself is also computationally efficient.

\section{Multiplayer General-sum Markov Games}
\label{sec:multiplayer}

In multiplayer general-sum games, finding Nash equilibria is computationally hard in general (which is technically PPAD-complete \cite{daskalakis2013complexity}). In this section, we focus on finding two commonly-used alternative notions of equilibria in the game theory---coarse correlated equilibria (CCE), and correlated equilibria (CE). Both are relaxed notions of Nash equilibira.

\subsection{Finding coarse correlated equilibria}

The algorithm for finding CCE is again running V-learning (Algorithm \ref{algorithm:V}) independently for each agent $j$ with learning rate $\alpha_t$ (as specified in \eqref{eq:hyper_V}) and bonus $\{\beta_{j, t}\}$ (to be specified later). The major difference from the case of finding Nash equilibria is that CCE and CE require the output policy to be joint correlated policy.
We achieve this correlation by feeding the same random seed to all agents at the very beginning when they execute the output policy according to Algorithm \ref{algorithm:V-policy}. That is, while training can be done in the fully decentralized fashion, we require one round of communication at the beginning of the execution to broadcast the shared random seed. After that, each agent can simply execute her own output policy independently. During the execution, since the states visited are shared among all agents, shared random seed allows the same index $i$ to be sampled across all agents in the Step \ref{line:sample_output} of Algorithm \ref{algorithm:V-policy} at every step. We denote this correlated joint output policy as $\hat{\pi} = \hat{\pi}_1 \odot \ldots \odot \hat{\pi}_m$.

We remark that to specify a correlated policy in general, we need to specify the probability for taking all action combinations $(a_1, \ldots, a_m)$ for each $(s, h)$. This requires at least $\Omega(HS\prod_{j=1}^m A_j)$ space, which grows exponentially with the number of agents $m$. The way V-learning specifies the joint policy only requires agents to store their own intermediate counters and policies computed during training. This only takes a total of $\cO(HSK(\sum_{j=1}^m A_j))$ space, which scales only linearly  with the number of agents. Our approach dramatically improve over the former approach in space complexity when the number of agents is large.

We now present the guarantees for V-learning to learn a CCE as follows. Let $A = \max_{j\in[m]} A_j$.




\begin{theorem}
    \label{thm:CCE_main_result} Suppose subroutine \textsc{Adv\_Bandit\_Update}  satisfies Assumption \ref{ass:external_regret}. For any $\delta \in (0, 1)$ and $K \in \N$, let $\iota = \log(mHSAK/\delta)$. Choose learning rate $\alpha_t$ according to \eqref{eq:hyper_V} and bonus $\{\beta_{j, t}\}_{t=1}^K$ of the $j\th$ player so that $\sum_{i=1}^t\alpha_t^i\beta_{j,i}=\Theta(H\AVGReg(A_j, t, \iota)+\sqrt{H^3\iota/t})$ for any $t \in [K]$. Then, with probability at least $1-\delta$, after all the players running Algorithm \ref{algorithm:V} for $K$ episodes, let $\hat{\pi}_j$ be the output policy by Algorithm \ref{algorithm:V-policy} for the $j\th$ player, then we have the joint policy $\hat{\pi} = \hat{\pi}_1 \odot \ldots \odot \hat{\pi}_m$ satisfies
    \begin{equation*}
    \max_{j\in[m]}[V_{j,1}^{\dag ,\hat{\pi}_{-j}}(s_1)-V_{j,1}^{\hat{\pi}}(s_1)] \le \cO( (H^2S/K) \cdot \CUMReg(A, K/S, \iota)+\sqrt{H^5S\iota/K}).
    \end{equation*}
    When instantiating \textsc{Adv\_Bandit\_Update} by FTRL (Algorithm \ref{algorithm:FTRL}), we can choose  $\beta_{j,t} = c \cdot \sqrt{H^3 A_j \iota/t}$ for some absolute constant $c$, which leads to $\max_{j\in[m]}[V_{j,1}^{\dag ,\hat{\pi}_{-j}}(s_1)-V_{j,1}^{\hat{\pi}}(s_1)] \le \cO(\sqrt{H^5SA\iota/K} )$.
\end{theorem}

Theorem \ref{thm:CCE_main_result} claims that, to find an $\epsilon-$approximate CCE, V-learning only needs to use a number of episodes $K=\tlO(H^5SA/\epsilon^2)$, where $A = \max_{j\in[m]} A_j$. This is in sharp contrast to the prior results for multiplayer general-sum Markov games, which use value iteration based algorithms, and require at least $\Omega(H^4 S^2 (\prod_{i=1}^m A_i)/\epsilon^2)$ episodes \cite{liu2021sharp}. As a result, V-learning is the first algorithm that breaks the curse of multiagents for finding CCE in Markov games.


\subsection{Finding correlated equilibria}
The algorithm for finding CE is almost the same as the algorithm for finding CCE except that we now require a different \textsc{Adv\_Bandit\_Update} subroutine, which has the following high probability weighted \emph{swap regret} guarantee.


\begin{assumption} \label{ass:swap_regret}
For any $t \in \N$ and any $\delta >0$, with probability at least $1-\delta$, we have
\begin{equation}\label{eq:swap_reg}
\max_{\psi \in \Psi}\sum_{i=1}^t \alpha_t^{i} [\la \theta_i, l_i \ra-\la \psi \diamond\theta_i, l_i\ra ] \le \AVGSwapReg(B, t, \log(1/\delta)).
\end{equation}
We assume the existence of an upper bound $\CUMSwapReg(B, t, \log(1/\delta)) \ge \sum_{t'=1}^t \AVGSwapReg(B, t', \log(1/\delta))$where (i) $\AVGSwapReg(B, t, \log(1/\delta))$ is non-decreasing in $B$ for any $t, \delta$; (ii) $\CUMSwapReg(B, t, \log(1/\delta))$ is concave in $t$ for any $B$, $\delta$.
\end{assumption}

Here $\Psi$ denotes the set $\{\psi: \cB \rightarrow \cB\}$ which consists of all maps from actions in $\cB$ to actions in $\cB$. Meanwhile, for any $\theta \in \Delta_{\cB}$, the term $\psi \diamond \theta \in \Delta_{\cB}$ denotes the distribution over actions where $\psi \diamond \theta (b) = \sum_{b':\psi(b') = b} \theta (b') $. We note that bounded swap regret is a stronger requirement compared to bounded external regret as in \eqref{eq:external_reg}, since by maximizing over a subset of functions in $\Psi$ which map all actions in $\cB$ to one single action, we recover the external regret by \eqref{eq:swap_reg}.

Assumption \ref{ass:swap_regret} can be satisfied by modifying many existing algorithms with external regret to the swap regret setting. In particular, we prove that the Follow-the-Regularized-Leader for swap regret (FTRL\_swap) algorithm (Algorithm \ref{algorithm:FTRL_swap}) satisfies  Assumption \ref{ass:swap_regret} with bounds $\AVGSwapReg(B, t,\log(1/\delta))\le \cO(B\sqrt{H\log(B/\delta) /t})$ and $\CUMSwapReg(B, t,\log(1/\delta))\le \cO(B\sqrt{Ht\log(B/\delta)})$. Both bounds have one extra $\sqrt{B}$ factor comparing to the counterparts in external regret. We refer readers to Appendix \ref{sec:bandit_swap} for more details.


We now present the guarantees for V-learning to learn a CCE as follows. Let $A = \max_{j\in[m]} A_j$.

\begin{theorem}
    \label{thm:CE_main_result} Suppose subroutine \textsc{Adv\_Bandit\_Update}  satisfies Assumption \ref{ass:swap_regret}. For any $\delta \in (0, 1)$ and $K \in \N$, let $\iota = \log(mHSAK/\delta)$. Choose learning rate $\alpha_t$ according to \eqref{eq:hyper_V} and bonus $\{\beta_{j, t}\}_{t=1}^K$ of the $j\th$ player so that $\sum_{i=1}^t\alpha_t^i\beta_{j,i}=\Theta(H\AVGSwapReg(A_j, t, \iota)+\sqrt{H^3\iota/t})$ for any $t \in [K]$. Then, with probability at least $1-\delta$, after all the players running Algorithm \ref{algorithm:V} for $K$ episodes, let $\hat{\pi}_j$ be the output policy by algorithm \ref{algorithm:V-policy} for the $j\th$ player, then we have the joint policy $\hat{\pi} = \hat{\pi}_1 \odot \ldots \odot \hat{\pi}_m$ satisfies
    \begin{equation*}
        \max_{j\in[m]}\max_{\phi_j}[V_{j,1}^{\phi_j\diamond\hat{\pi}}(s_1)-V_{j,1}^{\hat{\pi}}(s_1)] \le \cO( (H^2S/K)\cdot\CUMSwapReg(A, K/S, \iota) +\sqrt{SH^5\iota/K}).
    \end{equation*}
  When instantiating \textsc{Adv\_Bandit\_Update} by FTRL\_swap (Algorithm~\ref{algorithm:FTRL_swap}), we can choose  $\beta_{j,t} = c \cdot A_j\sqrt{H^3 \iota/t}$ for some absolute constant $c$, which leads to 
  $\max_{j\in[m]}\max_{\phi_j}[V_{j,1}^{\phi_j\diamond\hat{\pi}}(s_1)-V_{j,1}^{\hat{\pi}}(s_1)] \le \cO(A\sqrt{H^5S\iota/K} )$.

\end{theorem}
Theorem \ref{thm:CE_main_result} claims that, to find an $\epsilon-$approximate CE, V-learning only needs to use a number of episodes $K=\tlO(H^5SA^2/\epsilon^2)$, where $A = \max_{j\in[m]} A_j$. It has an extra $A$ multiplicative factor comparing to the sample complexity of finding CCE, since CE is a subset of CCE thus finding CE is expected to be more difficult. Nevertheless, the sample complexity presented here is far better than value iteration based algorithm, which requires at least $\Omega(H^4 S^2 (\prod_{i=1}^m A_i)/\epsilon^2)$ episodes for finding CE \cite{liu2021sharp}. V-learning is also the first algorithm that breaks the curse of multiagents for finding CE in Markov games.





\section{Monotonic V-Learning}
\label{sec:monotone}

In the previous sections, we present the V-learning algorithm whose output policy (Algorithm \ref{algorithm:V-policy}) is a nested mixture of Markov polices. Storing such a output policy requires $\cO(HSA_jK)$ space for the $j\th$ player. In Section \ref{sec:multiplayer}, we argue this approach has a significant advantage over directly storing a general correlated policy when the number of agents is large. Nevertheless, this space complexity can be undesirable when the number of agents is small.

In this section, we introduce a simple monotonic techique to V-learning, which allows each agent to output a Markov policy when finding Nash equilibria in the two-player zero-sum setting. Storing a Markov policy only takes $\cO(HSA_j)$ space for the $j\th$ player. A similar result for the single-agent setting can be immediately obtained by setting the second player in the Markov game to be a dummy player with only a single action to choose from.

\paragraph{Monotonic update} Monotonic V-learning is almost the same as V-learning with only the Line \ref{line:monotone} in Algorithm \ref{algorithm:V} changed to 
\begin{equation}
\label{eq:monotone_update}
    V_h(s_h) \setto \min\{H+1-h, \tilde{V}_h(s_h),V_h(s_h)\},
\end{equation}
This step guarantees $V_h(s_h)$ to monotonically decrease at each step. This is helpful because in two-player zero-sum Markov games, all Nash equilibria share a unique value which we denote as $V^\star$. By design, we can prove that the V-values maintained in V-learning are high probability upper bounds of $V^\star$ (Lemma \ref{lem:mono_order}). 
This monotonic update allows our V-value estimates to always get closer to $V^\star$ after each update, which improves the accuracy of our V-value estimates. 

\paragraph{Markov output policy} For an arbitrary fixed $(s,h)\in \mathcal{S}\times [H]$, let $t_1$ be the last episode when the value ${V}_{1,h}(s)$ is updated (i.e., strictly decreases), and let $t_2$ be the last episode when the value ${V}_{2,h}(s)$ is updated. Then the output policy for this $(s, h)$ has the following form.
\begin{equation}\label{eq:monotone_output}
\tilde\pi_{1,h}(\cdot|s):=\sum_{i=1}^{t_2}\alpha_{t_2}^i\pi_{1,h}^{k^i}(\cdot|s),\quad \tilde\pi_{2,h}(\cdot|s):=\sum_{i=1}^{t_1}\alpha_{t_1}^i\pi_{2,h}^{k^i}(\cdot|s),
\end{equation}
where $k^i$ denotes the index of episode when state $s$ is visited at step $h$ is visited for the $i^{\rm th}$ time. Recall that $\pi_{j,h}^{k}(\cdot|s)$ is the policy maintained by the $j\th$ player at the beginning of the $k\th$ episode when she runs V-learning. That is, the new output policy is simply the weighted average of policies computed in the V-learning at each $(s, h)$ pair. Clearly, the policies $\tilde{\pi}_1$ and $\tilde{\pi}_2$ defined by \eqref{eq:monotone_output} are Markov policies.

We remark that although the execution of $\tilde{\pi}_1$ and $\tilde{\pi}_2$ can be fully decentralized, in \eqref{eq:monotone_output} the computation of $\tilde\pi_{1,h}(\cdot|s)$ depends on $t_2$ while the computation of $\tilde\pi_{2,h}(\cdot|s)$ depends on $t_1$.
That is, two players need to communicate at the end the indexes of the most recent episodes when their $V$-values are updated. As a result, monotonic V-learning is not fully decentralized.

\begin{theorem}
    \label{thm:2p0s_monotone} Monotonic V-learning with output policy $\tilde{\pi} = \tilde{\pi}_1 \times \tilde{\pi}_2$ as specified by \eqref{eq:monotone_output} has the same theoretical guarantees as Theorem \ref{thm:2p0s_main_result} with the same choices of hyperparamters.
\end{theorem}

Theorem \ref{thm:2p0s_monotone} asserts that V-learning can be modified to output Markov policies when finding Nash equilibria of two-player zero-sum Markov games. As a special case, the same technique and results directly apply to the single-agent setting.

\section{Conclusion}
\label{sec:conclusion}
In this paper, we develop the first \emph{decentralized} algorithm that breaks the \emph{curse of multiagents} for learning general Markov games. Behind this new result is a new class of single-agent RL algorithms---V-learning, which converts any adversarial bandit algorithm with suitable regret guarantees into a RL algorithm. A remarkable advantage of V-learning is its effortless extension to the multiagent setting while having much preferred theoretical guarantees over existing methods: by simply running V-learning independently for all agents, we find Nash equilibria (for two-player zero-sum games), CCE, and CE in a number of samples that scales with only $\max_{j\in[m]} A_j$, in contrast to existing algorithms whose number of samples scales with $\prod_{j\in[m]} A_j$. 





\bibliographystyle{plainnat}
\bibliography{ref}

\clearpage

\appendix

\section{Notations and Basic Lemmas}
\label{app:notation}

\renewcommand{\a}{\bm{a}}

\subsection{Notations}

In this subsection, we introduce some notations that will be frequently used in appendixes. Recall that we use $V^k, N^k, \pi^k$ to denote the value, counter and policy maintained by V-learning algorithm at \emph{the beginning} of the episode $k$.

We also introduce a new policy $\hat \pi_h^k$ for a single agent (defined by its execution in Algorithm \ref{algorithm:V-sampling}), which can be viewed as a part of the output policy in Algorithm \ref{algorithm:V-policy}. 
The definition of $\hat \pi_h^k$ is very similar to $\hat \pi$ except two differences: (1) $\hat \pi_h^k$ is a policy for step $h,\ldots,H$ while $\hat\pi$ is a policy for step $1,\ldots,H$; (2) in $\hat\pi$ the initial value of $k$ is sampled uniformly at random from $[K]$ at the very beginning while in $\hat \pi_h^k$ the initial value of $k$ is given.

We remark that $\hat \pi_h^k$ is a non-Markov policy that does not depends on history before to the $h\th$ step. In symbol, we can express this class of policy as $\pi_j \defeq \{ \pi_{j,h'}: \Omega \times (\cS \times \cA)^{h'-h}\times \cS \rightarrow \cA_j \}_{h'=h}^H$. We call this class of policy the \emph{policy starting from the $h\th$ step}, and denote it as $\Pi_h$. Similar to Section \ref{sec:prelim}, we can also define joint policy $\pi = \pi_1 \odot \ldots \odot \pi_m$ and product policy $\pi = \pi_1 \times \ldots \times \pi_m$ for policies in $\Pi_h$. We can also define value $V^{\pi}_h(s)$ for joint policy $\pi \in \Pi_h$ as
\begin{equation*} 
 \textstyle V^{\pi}_{i, h}(s) \defeq \E_{\pi}\left[\left.\sum_{h' =
        h}^H r_{i,h'}(s_{h'}, \bm{a}_{h'}) \right| s_h=s \right].
\end{equation*}
This allows us to define the corresponding best response of $\pi_{-i}$ as the maximizer of $\max_{\pi'_i \in \Pi_h} V_{i,h}^{\pi'_i \times \pi_{-i}}(s)$. We also denote this maximum value as $V_{i,h}^{\dagger, \pi_{-i}}(s)$. We define the strategy modification for policies starting from the $h\th$ step as
$\phi_i:=\{\phi_{i,h'}: (\cS \times \cA)^{h'-h}\times \cS \times \cA_i \rightarrow \cA_i\}_{h'=h}^H$, and denote the set of such strategy modification as $\Phi_h$.

Finally, for simplicity of notation, we  define two operators $\P$ and $\D$ as follows:
\begin{equation}
	\begin{cases}
	\P_h[V](s,\a) = \E_{s'\sim\P_h(\cdot \mid s,\a)}[V(s')],\\
	\D_{\pi}[Q](s) = \E_{\a\sim\pi(\cdot\mid s)} [Q(s,\a)],
	\end{cases}
\end{equation}
for any value function $V$, $Q$ and any one-step Markov policy $\pi$.

\begin{algorithm}[t]
   \caption{\textsc{Executing Policy $\hat{\pi}_{h}^{k}$}}
   \label{algorithm:V-sampling}
   \begin{algorithmic}[1]
      \FOR{step $h'=h, h+1, \dots,H$}
      \STATE observe $s_{h'}$, and set $t \setto N^{k}_{h'}(s_{h'})$.
      \STATE set $k \setto k^i_{h'}(s_{h'})$, where $i\in [t]$ is sampled with probability $\alpha^i_t$. 
      \STATE take action $a_{h'} \sim \pi^{k}_{h'}(\cdot|s_{h'})$.
      \ENDFOR
   \end{algorithmic}
\end{algorithm}

\subsection{Basic lemmas}
We first present a proposition which clarify the relations among the three different kind of equilibria.
In particular, we show that, similar to strategic games, we also have Nash $\subset$ CE $\subset$ CCE in Markov games.
\begin{proposition}[Nash $\subset$ CE $\subset$ CCE]\label{prop:nash-ce-cce}
In Markov games, any $\epsilon$-approximate Nash equilibrium is an $\epsilon$-approximate CE, and any $\epsilon$-approximate CE is an $\epsilon$-approximate CCE.
\end{proposition}
\begin{proof}
We prove two claims separately. 

For Nash $\subset$ CE, let $\pi=\pi_1\times\pi_2\times\cdots\pi_m$ be an $\epsilon$-approximate Nash equilibrium, then
\begin{align*}
 \max_{\phi_i} V_{i, 1}^{(\phi_i \diamond \pi_i) \times \pi_{-i}}(s_1)
\stackrel{(a)}{=} \max_{\pi_i'}V_{i, 1}^{\pi_i' \times \pi_{-i}}(s_1) \stackrel{(b)}{\le}
V_{i, 1}^{\pi}(s_1)+\epsilon,
\end{align*}
Step (a) is because that $\pi$ is a product policy, where the randomness of different agents are completely independent. In this case, maximizing over strategy modification $\phi_i$ is equivalent to maximizing over a new independent policy. Step (b) directly follows from $\pi$ being an $\epsilon$-approximate Nash equilibrium. By definition, this proves that $\pi$ is also an $\epsilon$-approximate CE.

For CE $\subset$ CCE, let $\pi=\pi_1\odot\pi_2\odot\cdots\pi_m$ be an $\epsilon$-approximate CE, then we have
\begin{align*}
\max_{\pi_i'}V_{i, 1}^{\pi_i' \times \pi_{-i}}(s_1)
\stackrel{(c)}{\le}
\max_{\phi_i} V_{i, 1}^{(\phi_i \diamond \pi_i) \odot \pi_{-i}}(s_1) 
\stackrel{(d)}{\le}
V_{i, 1}^{\pi}(s_1)+\epsilon,
\end{align*}
Step (c) is because by definition of strategy modification $\phi_i:=\{\phi_{i,h}: (\cS \times \cA)^{h-1}\times \cS \times \cA_i \rightarrow \cA_i\}$, we can consider a subset of strategy modification $\phi'_i:=\{\phi'_{i,h}: (\cS \times \cA)^{h-1}\times \cS \rightarrow \cA_i\}$ which modifies the policy ignoring whatever the action $\pi_i$ takes. It is not hard to see that maxmizing over the strategy modification in this subset is equivalent to maximizing over a new independent policy $\pi_i'$. Therefore, maximizing over all strategy modification is greater or equal to maximizing over $\pi_i'$. Finally, step (d) follows from $\pi$ being an $\epsilon$-approximate CE. By definition, this proves that $\pi$ is also an $\epsilon$-approximate CCE.
\end{proof}


Next, we present some basic lemmas that will be used in the proofs of different theorems.
We start by introducing some useful properties of sequence $\{\alpha_t^i\}$ defined in \eqref{eq:hyper_V}.

\begin{lemma}(\citep[Lemma 4.1]{jin2018q},\citep[Lemma 2]{tian2021online})
\label{lem:step_size}
The following properties hold for $\alpha _{t}^{i}$:
\begin{enumerate}
\item \label{lem:lr_hoeffding}
$\frac{1}{\sqrt{t}} \le \sum_{i=1}^t \frac{\alpha^i_t}{\sqrt{i}} \le \frac{2}{\sqrt{t}}$ and $\frac{1}{t} \le \sum_{i=1}^t \frac{\alpha^i_t}{i} \le \frac{2}{t}$ for every $t \ge 1$.
\item \label{lem:lr_property0}
$\max_{i\in[t]} \alpha^i_t \le \frac{2H}{t}$ and $\sum_{i=1}^t (\alpha^i_t)^2 \le \frac{2H}{t}$ for every $t \ge 1$.
\item \label{lem:lr_property}
$\sum_{t=i}^\infty \alpha^i_t = 1 + \frac{1}{H}$ for every $i \geq 1$.
\end{enumerate}
\end{lemma}

Finally, we have the following lemma which express the $\tilde{V}$ maintained in V-learning in the form of weighted sum of earlier updates. 

\begin{lemma}
\label{lem:2p0s-mono-V-relation}
Consider an arbitrary fixed $(s,h,k)$ tuple.
   Let $t=N_{h}^{k}\left( s \right) $ denote the number of times $s$ is visited at step $h$ at the beginning of episode $k$, and suppose $s$ was previously visited at episodes $k^1,\ldots, k^t < k$ at the $h$-th step. Then the two V-values $\tilde V$ and $V$ in Algorithm \ref{algorithm:V} satisfy the following equation:
   \begin{equation}
      \tilde{V}_{j,h}^{k}(s) = 
      \alpha _{t}^{0}(H-h+1)+ \sum_{i=1}^t{\alpha _{t}^{i}\left[ r_{j,h}(s,\a_h^{k^i})+ {V}_{j,h+1}^{k^i}(s_{h+1}^{k^i}) + {\beta}_{j,i} \right]}, \quad j\in[m].
   \end{equation}
\end{lemma}
\begin{proof}
	The proof  follows directly from the update rule in Line \ref{line:update_V} Algorithm \ref{algorithm:V}. 
	Note that $\alpha _{t}^{0}$ is equal to zero for any $t>1$ and equal to one for $t=0$.
\end{proof}

\section{Proofs for Computing CCE in General-sum MGs}

\newcommand{\uhat}{\underaccent{\tilde}}
In this section, we give complete proof of Theorem~\ref{thm:CCE_main_result}. To avoid repeatedly state the condition of Theorem~\ref{thm:CCE_main_result} in each lemma, we will use
\begin{itemize}
  \item Condition of the adversarial bandit sub-procudure (Assumption~\ref{ass:external_regret}) and
  \item Set the bonus $\{\beta_{j, t}\}_{t=1}^K$ of the $j\th$ player so that $\sum_{i=1}^t\alpha_t^i\beta_{j,i}=\Theta(H\AVGReg(A_j, t, \iota)+\sqrt{H^3\iota/t})$ for any $t \in [K]$.
\end{itemize}
throughout the whole section.

The following Lemma is a direct consequence of Assumption \ref{ass:external_regret}, which will play an important role in our later analysis.

\begin{lemma}
 \label{lem:2p0s-momo-local-regret}
 Under Assumption \ref{ass:external_regret}, the following event is true with probability at least $1-\delta$: for any $(s,h,k) \in \cS \times [H] \times [K]$,  let $t=N_{h}^{k}\left( s \right) $ and suppose $s$ was previously visited at episodes $k^1,\ldots, k^t < k$ at the $h$-th step, then for all $j\in[m]$
   $$
\underset{\mu}{\max}\sum_{i=1}^t{\alpha _{t}^{i}\D_{\mu \times \pi_{ -j,h}^{k^i}} \left( r_{j,h} +\P_h {V}_{j,h+1}^{k^i} \right)\left( s \right) }
-\sum_{i=1}^t{\alpha _{t}^{i}\D_{\pi_{h}^{k^i}} \left( r_{j,h} +\P_h {V}_{j,h+1}^{k^i} \right)\left( s \right)}\le
H \AVGReg(A_j,t,\iota),
$$
where $\iota =\log(mHSAK/\delta)$.
\end{lemma}
\begin{proof}
By Assumption \ref{ass:external_regret} and the adversarial bandit update step in Algorithm \ref{algorithm:V}, we have that with probability at least $1-\delta$, for any $(s,h,k) \in \cS \times [H] \times [K]$,  
 $$
\underset{\mu}{\max}
\sum_{i=1}^t{\alpha _{t}^{i}\D_{\pi_{h}^{k^i}} \left(\frac{H- r_{j,h} -\P_h {V}_{j,h+1}^{k^i}}{H}\right)\left( s \right)}
-
\sum_{i=1}^t{\alpha _{t}^{i}\D_{\mu \times \pi_{ -j,h}^{k^i}} \left(\frac{H- r_{j,h} -\P_h {V}_{j,h+1}^{k^i}}{H}\right)\left( s \right) }
\le
 \AVGReg(A_j,t,\iota),
$$
which implies the desired result by simple algebraic transformation. 
\end{proof}

Then we show $V$ is actually an optimistic estimation of the value function 
of player $j$'th best response to  the output policy.

\begin{lemma}[Optimism]
    \label{lem:opt_CCE}
    For any $\delta \in (0,1]$, with probability at least $1-\delta$, for any $(s,h,k,j) \in \cS \times [H] \times [K]\times [m]$,
    $
  V_{j,h}^{k}(s) \ge  V_{j,h}^{\dag ,\hat{\pi}_{-j,h}^{k}}(s)
    $.
  \end{lemma}

  \begin{proof}[Proof of Lemma~\ref{lem:opt_CCE}]
     We prove by backward induction.
      The claim is trivially satisfied for $h=H+1$. Suppose it is true for $h+1$, consider a fixed state $s$.
      It suffices to show
      $\tilde{V}_{j,h}^{k}(s)\ge V_{j,h}^{\dag ,\hat{\pi}_{-j,h}^{k}}(s)$ because
      $V_{j,h}^{k}(s) = \min\{\tilde{V}_{j,h}^{k}(s),H-h+1\}$.
      Let $t=N_{h}^{k}\left( s \right) $ and suppose $s$ was previously visited at episodes $k^1,\ldots, k^t < k$ at the $h$-th step. Then using Lemma~\ref{lem:2p0s-mono-V-relation},
  \begin{align*}
    \tilde{V}_{j,h}^{k}(s)&=\alpha _{t}^{0}H+\sum_{i=1}^t{\alpha _{t}^{i}\left[ r_{j,h}(s,\bm{a}_h^{k^i})+V_{j,h+1}^{k^i}(s_{h+1}^{k^i}) +\beta_{j,i} \right]}
  \\
  &\overset{\left( i \right)}{\ge} \sum_{i=1}^t{\alpha _{t}^{i}\D_{\pi _{h}^{k^i} }\left( r_{j,h} +\P_hV_{j,h+1}^{k^i} \right)\left( s \right)}+\sum_{i=1}^t{\alpha _{t}^{i}\beta_{j,i}}-\cO\left(\sqrt{\frac{H^3\iota}{t}}\right)
  \\
  &\overset{\left( ii \right)}{\ge} \underset{\mu}{\max}\sum_{i=1}^t{\alpha _{t}^{i}\D_{\mu \times \pi _{-j,h}^{k^i} }\left( r_{j,h} +\P_hV_{j,h+1}^{k^i} \right)\left( s \right)}+\sum_{i=1}^t{\alpha _{t}^{i}\beta_{j,i}}-\cO\left(\sqrt{\frac{H^3\iota}{t}}\right)-H \AVGReg(A_j,t,\iota)
  \\
  &\overset{\left( iii \right)}{\ge} \underset{\mu}{\max}\sum_{i=1}^t{\alpha _{t}^{i}\D_{\mu \times \pi _{-j,h}^{k^i} }\left( r_{j,h} +\P_hV_{j,h+1}^{k^i} \right)\left( s \right)}
  \\
  &\overset{\left( iv \right)}{\ge} \underset{\mu}{\max}\sum_{i=1}^t{\alpha _{t}^{i}\D_{\mu \times \pi _{-j,h}^{k^i} }\left( r_{j,h} +\P_h V_{j,h+1}^{\dag ,\hat{\pi}_{-j,h+1}^{k^i}} \right)\left( s \right)} \stackrel{(v)}{\ge} V_{j,h}^{\dag ,\hat{\pi}_{-j,h}^{k}}(s)
  \end{align*}
  where $(i)$ is by martingale concentration and Lemma~\ref{lem:lr_property0}, $(ii)$ is by Lemma~\ref{lem:2p0s-momo-local-regret}, $(iii)$ is by the definition of ${\beta}_{j,i}$, and $(iv)$ is by induction hypothesis. 

  Finally, we remark that $(v)$ is not directly from Bellman equation since $\hat{\pi}_{-j,h}^{k}$ is non-Markov policy, and the best reponse of a non-Markov policy is not necessary a Markov policy. We prove $(v)$ as follows. Recalls definitions for policies in $\Pi_h$ as in Appendix \ref{app:notation}, by the definition, we have 
  \begin{align*}
   V_{j,h}^{\dag ,\hat{\pi}_{-j,h}^{k}}(s) =& \max_{\mu \in \Pi_h} V_{j,h}^{\mu \times\hat{\pi}_{-j,h}^{k}} \\
  \stackrel{(a)}{=}& \max_{\mu_h} \max_{\mu_{(h+1):H}} \sum_{i=1}^t{\alpha _{t}^{i}\E_{\bm{a} \sim \mu_h \times \pi _{-j,h}^{k^i} }\left( r_{j,h}(s, \bm{a}) + \E_{s'} V_{j,h+1}^{\mu_{(h+1):H} ,\hat{\pi}_{-j,h+1}^{k^i}}(s, \bm{a}, s') \right)}\\
  \stackrel{(b)}{\le}& \max_{\mu_h} \sum_{i=1}^t{\alpha _{t}^{i}\E_{\bm{a} \sim \mu_h \times \pi _{-j,h}^{k^i} }\left( r_{j,h}(s, \bm{a}) + \E_{s'} \max_{\mu_{(h+1):H}} V_{j,h+1}^{\mu_{(h+1):H} ,\hat{\pi}_{-j,h+1}^{k^i}}(s, \bm{a}, s') \right)}\\
  \stackrel{(c)}{=}& 
  \max_{\mu_h}\sum_{i=1}^t\alpha _{t}^{i} \E_{\bm{a} \sim \mu_h \times \pi _{-j,h}^{k^i} }\left( r_{j,h}(s, \bm{a}) + 
  \E_{s'} V_{j,h+1}^{\dagger, \hat{\pi}_{-j,h+1}^{k^i}}(s') \right)\\
  =& \underset{\mu}{\max}\sum_{i=1}^t{\alpha _{t}^{i}\D_{\mu \times \pi _{-j,h}^{k^i} }\left( r_{j,h} +\P_h V_{j,h+1}^{\dag ,\hat{\pi}_{-j,h+1}^{k^i}} \right)\left( s \right)}
  \end{align*}
  where $V_{j,h+1}^\pi(s, \bm{a}, s')$ for policy $\pi \in \Pi_h$ is defined as:
  \begin{equation*} 
  \textstyle V^{\pi}_{i, h+1}(s, \bm{a}, s') \defeq \E_{\pi}\left[\left.\sum_{h' =
        h+1}^H r_{h'}(s_{h'}, \bm{a}_{h'}) \right| s_h=s, \bm{a}_h = \bm{a}, s_{h+1} = s'\right].
  \end{equation*}
  Step (a) uses the relation between $\hat{\pi}_{-j,h}^{k}$ and $\{\hat{\pi}_{-j,h+1}^{k^i}\}_i$. Step (b) pushes $\max$ inside summation and expectation. Step (c) is because the Markov nature of Markov game and that $\{\hat{\pi}_{-j,h+1}^{k^i}\}_i$ are policies that does not depend on history at step $h$, we know the maximization over $\mu_{(h+1):H}$ is achieved at policies in $\Pi_{h+1}$. This finishes the proof.
  \end{proof}

To proceed with the analysis, we  need to introduce two  pessimistic V-estimations $\uhat{V}$ and $\low{V}$ that are defined similarly as $\tilde{V}$ and $V$. Formally, let $t=N_{h}^{k}\left( s \right) $ denote the number of times $s$ is visited at step $h$ at the beginning of episode $k$, and suppose $s$ was previously visited at episodes $k^1,\ldots, k^t < k$ at the $h$-th step. Then
   \begin{equation}
    \label{eq:undertilde-def}
    \uhat{V}_{j,h}^{k}(s) = \sum_{i=1}^t{\alpha _{t}^{i}\left[ r_{j,h}(s,\a_h^{k^i})+ \low{V}_{j,h+1}^{k^i}(s_{h+1}^{k^i}) - {\beta}_{j,i} \right]},
   \end{equation}
  \begin{equation}
    \label{eq:low-def}
    \low{V}^{k}_{j,h}(s) = \max\{0, \uhat{V}^{k}_{j,h}(s)\},
  \end{equation}
  for any player $j \in [m]$ and $k \in [K]$. 
  We emphasize that $\uhat{V}$ and $\low{V}$ are defined only for the purpose of analysis. 
  Neither do they influence the decision made by each agent, nor do the agents need to maintain these quantities when running V-learning.

  Equipped with the lower estimations, we  are ready to lower bound $V_{j,h}^{\hat{\pi}_h^{k}}$.
  \begin{lemma}[Pessimism]
    \label{lem:pess_CCE}
    For any $\delta \in (0,1]$, with probability at least $1-\delta$, the following holds for any $(s,h,k,j) \in \cS \times [H] \times [K]\times [m]$ and any player $j$,
    $
  \low{V}_{j,h}^{k}(s) \le  V_{j,h}^{\hat{\pi}^{k}_{h}}(s)
    $.
  \end{lemma}

  \begin{proof}[Proof of Lemma~\ref{lem:pess_CCE}]
    We prove by backward induction. The claim is trivially satisfied for $h=H+1$. Suppose it is true for $h+1$, consider a fixed state $s$.
      It suffices to show
      $\uhat{V}_{j,h}^{k}(s)\le V_{j,h}^{\hat{\pi}^{k}_{h}}(s)$ because
      $\low{V}_{j,h}^{k}(s) = \max\{\uhat{V}_{j,h}^{k}(s),0\}$.
      Let $t=N_{h}^{k}\left( s \right) $ and suppose $s$ was previously visited at episodes $k^1,\ldots, k^t < k$ at the $h$-th step. Then by Equation~\ref{eq:undertilde-def},
  \begin{align*}
   \uhat{V}_{j,h}^{k}(s)&=\sum_{i=1}^t{\alpha _{t}^{i}\left[ r_{j,h}(s,\bm{a}_h^{k^i})+\low{V}_{j,h+1}^{k^i}(s_{h+1}^{k^i}) -{\beta}_{j,i} \right]}
  \\
  &\overset{\left( i \right)}{\le} \sum_{i=1}^t{\alpha _{t}^{i}\D_{\pi _{h}^{k^i} }\left( r_h +\P_h\low{V}_{j,h+1}^{k^i} \right)\left( s \right)}-\sum_{i=1}^t{\alpha _{t}^{i}{\beta}_{j,i}}+\cO\left(\sqrt{\frac{H^3\iota}{t}}\right)
  \\
  &\overset{\left( ii \right)}{\le} \sum_{i=1}^t{\alpha _{t}^{i}\D_{\pi _{h}^{k^i} }\left( r_h +\P_h\low{V}_{j,h+1}^{k^i} \right)\left( s \right)}
  \\
  &\overset{\left( iii \right)}{\le} \sum_{i=1}^t{\alpha _{t}^{i}\D_{\pi _{h}^{k^i} }\left( r_h +\P_h V_{j,h+1}^{\hat{\pi}_h^{k^i}} \right)\left( s \right)}
  \\
  &=V_{j,h}^{\hat{\pi}^{k}_h}(s)
  \end{align*}
  where $(i)$ is by martingale concentration, $(ii)$ is by the definition of ${\beta}_{j,i}$, and $(iii)$ is by induction hypothesis.
  \end{proof}

  To prove Theorem~\ref{thm:CCE_main_result}, it remains to bound the gap $\sum_{k=1}^K\max_{j}( V_{1,j}^{k}-\low{V}_{1,j}^{k} ) ( s_1 )$.


  \begin{proof}[Proof of Theorem~\ref{thm:CCE_main_result}]
    Consider player $j$, we define $\delta _{j,h}^{k}:= V_{j,h}^{k}(s_h^k)-\low{V}_{j,h}^{k}(s_h^k) \ge 0$. The non-negativity here is a simple consequence of the update rule and induction. We want to bound $\delta _{h}^{k}:= \max_{j} \delta _{j,h}^{k}$.
  Let $n_h^k =N_{h}^{k}\left( s_h^k \right) $ and suppose $s_h^k$ was previously visited at episodes $k^1,\ldots, k^{n_h^k} < k$ at the $h$-th step.
  Now by the update rule of $V_{j,h}^{k}(s_h^k)$ and $\low{V}_{j,h}^{k}(s_h^k)$,
  \begin{align*}
  \delta _{j,h}^{k}=& V_{j,h}^{k}(s_h^k)-\low{V}_{j,h}^{k}(s_h^k)
  \\
  \le &\alpha _{n_h^k}^{0}H+\sum_{i=1}^{n_h^k}{\alpha _{n_h^k}^{i}\left[\left( V_{j,h+1}^{k^i }-\low{V}_{j,h+1}^{k^i } \right)\left( s_{h+1}^{k^i} \right) +2\beta_{j,i} \right]}
  \\
  =&\alpha _{n_h^k}^{0}H+\sum_{i=1}^{n_h^k}{\alpha _{n_h^k}^{i}\delta _{j,h+1}^{k^i }}+\cO(H\AVGReg(A_j, n_h^k, \iota)+\sqrt{H^3\iota/n_h^k})
  \end{align*}
  where in the last step we have used $\sum_{i=1}^t\alpha_t^i\beta_{j,i}=\Theta(H\AVGReg(A_j, t, \iota)+\sqrt{H^3\iota/t})$.

  Now by taking maximum w.r.t. $j$ on both sides and notice $\AVGReg(B, t, \iota)$ is non-decreasing in $B$, we have
  \begin{align*}
    \delta _{h}^{k}\le &\alpha _{n_h^k}^{0}H+\sum_{i=1}^{n_h^k}{\alpha _{n_h^k}^{i}\delta _{h+1}^{k^i}}+\cO(H\AVGReg(A, n_h^k, \iota)+\sqrt{H^3\iota/n_h^k}).
    \end{align*}

  Summing the first two terms w.r.t. $k$,
  $$
  \sum_{k=1}^K{\alpha _{n_{h}^{k}}^{0}H}=\sum_{k=1}^K{H\mathbb{I}\left\{ n_{h}^{k}=0 \right\}}\le SH,
  $$
  $$
  \sum_{k=1}^K{\sum_{i=1}^{n_{h}^{k}}{\alpha _{n_{h}^{k}}^{i}\delta _{h+1}^{k^i}}}\overset{\left( i \right)}{\le} \sum_{k'=1}^K{\delta _{h+1}^{k'}\sum_{i=n_{h}^{k'}+1}^{\infty}{\alpha _{i}^{n_{h}^{k'}}}}\overset{\left( ii \right)}{\le} \left( 1+\frac{1}{H} \right) \sum_{k=1}^K{\delta _{h+1}^{k}}.
  $$
  where $(i)$ is by changing the order of summation and $(ii)$ is by Lemma~\ref{lem:step_size}. Putting them together,

  \begin{align*}
      \sum_{k=1}^K{\delta _{h}^{k}}=&\sum_{k=1}^K{\alpha _{n_{h}^{k}}^{0}H}+\sum_{k=1}^K{\sum_{i=1}^{n_{h}^{k}}{\alpha _{n_{h}^{k}}^{i}\delta _{h+1}^{k^i}}}+\sum_{k=1}^K{\cO(H\AVGReg(A, n_h^k, \iota)+\sqrt{H^3\iota/n_h^k})}
  \\
  \le& HS+\left( 1+\frac{1}{H} \right) \sum_{k=1}^K{\delta _{h+1}^{k}}+\sum_{k=1}^K{\cO(H\AVGReg(A, n_h^k, \iota)+\sqrt{H^3\iota/n_h^k})}
  \end{align*}
  Recursing this argument for $h \in [H]$ gives
     $$
     \sum_{k=1}^K{\delta _{1}^{k}}\le eSH^2+e\sum_{h=1}^H\sum_{k=1}^K{\cO(H\AVGReg(A, n_h^k, \iota)+\sqrt{H^3\iota/n_h^k})}
     $$

     By pigeonhole argument,
     \begin{align*}
     \sum_{k=1}^K{(H\AVGReg(A, n_h^k, \iota)+\sqrt{H^3\iota/n_h^k})}=&\cO\left( 1 \right) \sum_{s}\sum_{n=1}^{N_{h}^{K}\left( s \right)}\left(H\AVGReg(A, n, \iota)+\sqrt{\frac{H^3\iota}{n}}\right)\\
     \le&\cO\left( 1 \right) \sum_{s}{\left(H\CUMReg(A, N_{h}^{K}\left( s \right), \iota)+\sqrt{H^3N_{h}^{K}(s)\iota}\right)}\\
      \le& \cO\left( HS\CUMReg(A, K/S, \iota)+ \sqrt{H^3SK\iota}\right),
     \end{align*}
  where in the last step we have used concavity.

  Finally take the sum w.r.t. $h\in[H]$ we have
  \begin{equation*}
  \sum_{k=1}^K{\max_{j}[V_{1,j}^{k}-\low{V}_{1,j}^{k}](s_1)} \le \cO\left( H^2S\CUMReg(A, K/S, \iota)+ \sqrt{H^5SK\iota}\right),
  \end{equation*}
  which implies
  \begin{equation*}
    \max_{j\in[m]}[V_{j,1}^{\dag ,\hat{\pi}_{-j}}(s_1)-V_{j,1}^{\hat{\pi}}(s_1)] \le \cO( (H^2S/K) \cdot \CUMReg(A, K/S, \iota)+\sqrt{H^5S\iota/K}).\qedhere
    \end{equation*}
  \end{proof}




\section{Proofs for Computing CE in General-sum MGs}

In this section, we give complete proof of Theorem~\ref{thm:CE_main_result}. To avoid repeatedly state the condition of Theorem~\ref{thm:CE_main_result} in each lemma, we will use 
\begin{itemize}
  \item Condition of the adversarial bandit sub-procudure (Assumption~\ref{ass:swap_regret}) and
  \item Set the bonus $\{\beta_{j, t}\}_{t=1}^K$ of the $j\th$ player so that $\sum_{i=1}^t\alpha_t^i\beta_{j,i}=\Theta(H\AVGSwapReg(A_j, t, \iota)+\sqrt{H^3\iota/t})$ for any $t \in [K]$.
\end{itemize}
throughout the whole section.

We begin with a swap regret version of Lemma~\ref{lem:2p0s-momo-local-regret}.

\begin{lemma}
  \label{lem:swap-local-regret}
  The following event is true with probability at least $1-\delta$: for any $(s,h,k) \in \cS \times [H] \times [K]$,  let $t=N_{h}^{k}\left( s \right) $ and suppose $s$ was previously visited at episodes $k^1,\ldots, k^t < k$ at the $h$-th step, then for all $j\in[m]$
    $$
    \max_{\phi_j} \sum_{i=1}^t \alpha_t^i \D_{\phi_j\diamond\pi_{j,h}^{k^i}\times \pi_{-j,h}^{k^i}} \left[r_{j,h}+\P_hV_{j,h+1}^{k^i}\right](s) - \sum_{i=1}^t{\alpha _{t}^{i}\D_{\pi_{h}^{k^i}} \left( r_{j,h} +\P_h {V}_{j,h+1}^{k^i} \right)\left( s \right)}\le
 H \AVGSwapReg(A_j,t,\iota),
 $$
 where $\iota =\log(KHS/\delta)$.
 \end{lemma}

 \begin{proof}
  By Assumption \ref{ass:swap_regret} and the adversarial bandit update step in Algorithm \ref{algorithm:V}, we have that with probability at least $1-\delta$, for any $(s,h,k) \in \cS \times [H] \times [K]$,  
   \begin{align*}
   	\max_{\phi_j} \sum_{i=1}^t \alpha_t^i \D_{\phi_j\diamond\pi_{j,h}^{k^i}\times \pi_{-j,h}^{k^i}}\left( \frac{H-r_{j,h}+\P_hV_{j,h+1}^{k^i}}{H}\right) (s)& - \sum_{i=1}^t{\alpha _{t}^{i}\D_{\pi_{h}^{k^i}} \left( \frac{H-r_{j,h}+\P_hV_{j,h+1}^{k^i}}{H}\right) \left( s \right)}\\
  \le&
   \AVGSwapReg(A_j,t,\iota).
   \end{align*} 
  \end{proof}

  We begin with proving $V$ is actually an optimistic estimation of the value function under best response.
  
  \begin{lemma}[Optimism]
    \label{lem:optimism_CE}
    For any $\delta \in (0,1)$, with probability at least $1-\delta$, the following holds for any $(s,h,k,j) \in \cS \times [H] \times [K]\times [m]$,
    $
  V_{j,h}^{k}(s) \ge \max_{\phi_j}  V^{(\phi_j\diamond\hat{\pi}^k_{j,h})\odot\hat{\pi}^k_{-j,h} }_{j,h}(s)
    $.
  \end{lemma}

  \begin{proof}[Proof of Lemma~\ref{lem:optimism_CE}]
    We prove by backward induction.
      The claim is trivially satisfied for $h=H+1$. Suppose it is true for $h+1$, consider a fixed state $s$. 
      It suffices to show 
      $\tilde{V}_{j,h}^{k}(s)\ge \max_{\phi_j}  V^{(\phi_j\diamond\hat{\pi}^k_{j,h})\odot\hat{\pi}^k_{-j,h} }_{j,h}(s)$ because
      $V_{j,h}^{k}(s) = \min\{\tilde{V}_{j,h}^{k}(s),H-h+1\}$.
      Let $t=N_{h}^{k}\left( s \right) $ and suppose $s$ was previously visited at episodes $k^1,\ldots, k^t < k$ at the $h$-th step. Then using Lemma~\ref{lem:2p0s-mono-V-relation},
  \begin{align*}
    \tilde{V}_{j,h}^{k}(s)&=\alpha _{t}^{0}(H-h+1)+\sum_{i=1}^t{\alpha _{t}^{i}\left[ r_{j,h}(s,\bm{a}_h^{k^i})+V_{j,h+1}^{k^i}(s_{h+1}^{k^i}) +\beta_{j,i} \right]}
  \\
  &\overset{\left( i \right)}{\ge} \sum_{i=1}^t{\alpha _{t}^{i}\D_{\pi _{h}^{k^i} }\left( r_{j,h} +\P_hV_{j,h+1}^{k^i} \right)\left( s \right)}+\sum_{i=1}^t{\alpha _{t}^{i}\beta_{j,i}}-\cO\left(\sqrt{\frac{H^3\iota}{t}}\right)
  \\
  &\overset{\left( ii \right)}{\ge} \max_{\phi_j}\sum_{i=1}^t{\alpha _{t}^{i}\D_{(\phi_j\diamond\pi _{j,h}^{k^i}) \times \pi _{-j,h}^{k^i} }\left( r_h +\P_hV_{j,h+1}^{k^i} \right)\left( s \right)}+\sum_{i=1}^t{\alpha _{t}^{i}\beta_{j,i}}-\cO\left(\sqrt{\frac{H^3\iota}{t}}\right)-H \AVGReg(A_j,t,\iota)
  \\
  &\overset{\left( iii \right)}{\ge} \max_{\phi_j}\sum_{i=1}^t{\alpha _{t}^{i}\D_{(\phi_j\diamond\pi _{j,h}^{k^i}) \times \pi _{-j,h}^{k^i} }\left( r_h +\P_hV_{j,h+1}^{k^i} \right)\left( s \right)}
  \\
  &\overset{\left( iv \right)}{\ge} \max_{\phi_j}\sum_{i=1}^t{\alpha _{t}^{i}\D_{(\phi_j\diamond\pi _{j,h}^{k^i}) \times \pi _{-j,h}^{k^i} }\left( r_h +\P_h \max_{\phi'_j}  V^{(\phi_j'\diamond\hat{\pi}^{k^i}_{j,h+1})\odot\hat{\pi}^{k^i}_{-j,h+1}}_{j,h} \right)\left( s \right)}
  \\
  &\stackrel{(v)}{\ge} \max_{\phi_j}  V^{(\phi_j\diamond\hat{\pi}^k_{j,h})\odot\hat{\pi}^k_{-j,h} }_{j,h}(s)
  \end{align*}
  where $(i)$ is by martingale concentration and Lemma~\ref{lem:lr_property0}, $(ii)$ is by Lemma~\ref{lem:swap-local-regret}, $(iii)$ is by the definition of ${\beta}_{j,i}$, and $(iv)$ is by induction hypothesis. Finally, $(v)$ follows from a similar reasoning as in the proof of Lemma \ref{lem:opt_CCE}, which we omit here.
  \end{proof}
  
  We still need to lower bound $V_{j,h}^{\hat{\pi}_h^{k}}$. To do this, we estimate $\uhat{V}$ and $\low{V}$ defined by Equation~\ref{eq:undertilde-def} and Equation~\ref{eq:low-def}. These quantities are indeed the lower bounds we need.
  
  \begin{lemma}[Pessimism]
    \label{lem:pessimism_CE}
     For any $\delta \in (0,1)$, with probability at least $1-\delta$, the following holds for any $(s,h,k,j) \in \cS \times [H] \times [K]\times [m]$,
    $
  \low{V}_{j,h}^{k}(s) \le  V_{h}^{\hat{\pi}_h^{k}}(s)
    $.
  \end{lemma}
  
  \begin{proof}[Proof of Lemma~\ref{lem:pessimism_CE}]
    We prove by backward induction. The claim is trivially satisfied for $h=H+1$. Suppose it is true for $h+1$, consider a fixed state $s$. 
      It suffices to show 
      $\uhat{V}_{j,h}^{k}(s)\le V_{j,h}^{\hat{\pi}^{k}_{h}}(s)$ because
      $\low{V}_{j,h}^{k}(s) = \max\{\uhat{V}_{j,h}^{k}(s),0\}$.
      Let $t=N_{h}^{k}\left( s \right) $ and suppose $s$ was previously visited at episodes $k^1,\ldots, k^t < k$ at the $h$-th step. Then by equation~\eqref{eq:undertilde-def},
      \begin{align*}
        \uhat{V}_{j,h}^{k}(s)&=\sum_{i=1}^t{\alpha _{t}^{i}\left[ r_{j,h}(s,\bm{a}_h^{k^i})+\low{V}_{j,h+1}^{k^i}(s_{h+1}^{k^i}) -{\beta}_{j,i} \right]}
       \\
       &\overset{\left( i \right)}{\le} \sum_{i=1}^t{\alpha _{t}^{i}\D_{\pi _{h}^{k^i} }\left( r_h +\P_h\low{V}_{j,h+1}^{k^i} \right)\left( s \right)}-\sum_{i=1}^t{\alpha _{t}^{i}{\beta}_{j,i}}+\cO\left(\sqrt{\frac{H^3\iota}{t}}\right)
       \\
       &\overset{\left( ii \right)}{\le} \sum_{i=1}^t{\alpha _{t}^{i}\D_{\pi _{h}^{k^i} }\left( r_h +\P_h\low{V}_{j,h+1}^{k^i} \right)\left( s \right)}
       \\
       &\overset{\left( iii \right)}{\le} \sum_{i=1}^t{\alpha _{t}^{i}\D_{\pi _{h}^{k^i} }\left( r_h +\P_h V_{j,h+1}^{\hat{\pi}_h^{k^i}} \right)\left( s \right)}
       \\
       &=V_{j,h}^{\hat{\pi}^{k}_h}(s)
       \end{align*}
       where $(i)$ is by martingale concentration, $(ii)$ is by the definition of ${\beta}_{j,i}$, and $(iii)$ is by induction hypothesis.
  \end{proof}
  
  To prove Theorem~\ref{thm:CE_main_result}, it remains to bound the gap $\sum_{k=1}^K\max_{j}( V_{1,j}^{k}-\low{V}_{1,j}^{k} ) ( s_1 )$.

  \begin{proof}[Proof of Theorem~\ref{thm:CE_main_result}]
    Consider player $j$, we define $\delta _{j,h}^{k}:= V_{j,h}^{k}(s_h^k)-\low{V}_{j,h}^{k}(s_h^k) \ge 0$. The non-negativity here is a simple consequence of the update rule and induction. We want to bound $\delta _{h}^{k}:= \max_{j} \delta _{j,h}^{k}$.
  Let $n_h^k =N_{h}^{k}\left( s_h^k \right) $ and suppose $s_h^k$ was previously visited at episodes $k^1,\ldots, k^{n_h^k} < k$ at the $h$-th step. 
  Now by the update rule of $V_{j,h}^{k}(s_h^k)$ and $\low{V}_{j,h}^{k}(s_h^k)$,
  \begin{align*}
  \delta _{j,h}^{k}=& V_{j,h}^{k}(s_h^k)-\low{V}_{j,h}^{k}(s_h^k)
  \\
  \le &\alpha _{n_h^k}^{0}H+\sum_{i=1}^{n_h^k}{\alpha _{n_h^k}^{i}\left[\left( V_{j,h+1}^{k^i }-\low{V}_{j,h+1}^{k^i } \right)\left( s_{h+1}^{k^i} \right) +2\beta_{j,i} \right]}
  \\
  =&\alpha _{n_h^k}^{0}H+\sum_{i=1}^{n_h^k}{\alpha _{n_h^k}^{i}\delta _{j,h+1}^{k^i }}+\cO(H\AVGSwapReg(A_j, n_h^k, \iota)+\sqrt{H^3\iota/n_h^k})
  \end{align*}
  where in the last step we have used $\sum_{i=1}^t\alpha_t^i\beta_{j,i}=\Theta(H\AVGSwapReg(A_j, t, \iota)+\sqrt{H^3\iota/t})$.
  
  Now by taking maximum w.r.t. $j$ on both sides and notice $\AVGSwapReg(B, t, \iota)$ is non-decreasing in $B$, we have 
  \begin{align*}
    \delta _{h}^{k}\le &\alpha _{n_h^k}^{0}H+\sum_{i=1}^{n_h^k}{\alpha _{n_h^k}^{i}\delta _{h+1}^{k^i}}+\cO(H\AVGSwapReg(A, n_h^k, \iota)+\sqrt{H^3\iota/n_h^k}).
    \end{align*}

  Summing the first two terms w.r.t. $k$,
  $$
  \sum_{k=1}^K{\alpha _{n_{h}^{k}}^{0}H}=\sum_{k=1}^K{H\mathbb{I}\left\{ n_{h}^{k}=0 \right\}}\le SH,
  $$
  $$
  \sum_{k=1}^K{\sum_{i=1}^{n_{h}^{k}}{\alpha _{n_{h}^{k}}^{i}\delta _{h+1}^{k^i}}}\overset{\left( i \right)}{\le} \sum_{k'=1}^K{\delta _{h+1}^{k'}\sum_{i=n_{h}^{k'}+1}^{\infty}{\alpha _{i}^{n_{h}^{k'}}}}\overset{\left( ii \right)}{\le} \left( 1+\frac{1}{H} \right) \sum_{k=1}^K{\delta _{h+1}^{k}}.
  $$
  where $(i)$ is by changing the order of summation and $(ii)$ is by Lemma~\ref{lem:step_size}. Putting them together,
  
  \begin{align*}
      \sum_{k=1}^K{\delta _{h}^{k}}=&\sum_{k=1}^K{\alpha _{n_{h}^{k}}^{0}H}+\sum_{k=1}^K{\sum_{i=1}^{n_{h}^{k}}{\alpha _{n_{h}^{k}}^{i}\delta _{h+1}^{k^i}}}+\sum_{k=1}^K{\cO(H\AVGSwapReg(A, n_h^k, \iota)+\sqrt{H^3\iota/n_h^k})}
  \\
  \le& HS+\left( 1+\frac{1}{H} \right) \sum_{k=1}^K{\delta _{h+1}^{k}}+\sum_{k=1}^K{\cO(H\AVGSwapReg(A, n_h^k, \iota)+\sqrt{H^3\iota/n_h^k})}
  \end{align*}
  Recursing this argument for $h \in [H]$ gives
     $$
     \sum_{k=1}^K{\delta _{1}^{k}}\le eSH^2+e\sum_{h=1}^H\sum_{k=1}^K{\cO(H\AVGSwapReg(A, n_h^k, \iota)+\sqrt{H^3\iota/n_h^k})}
     $$
  
     By pigeonhole argument,
     \begin{align*}
     \sum_{k=1}^K{(H\AVGSwapReg(A, n_h^k, \iota)+\sqrt{H^3\iota/n_h^k})}=&\cO\left( 1 \right) \sum_{s}\sum_{n=1}^{N_{h}^{K}\left( s \right)}\left(H\AVGSwapReg(A, n, \iota)+\sqrt{\frac{H^3\iota}{n}}\right)\\
     \le&\cO\left( 1 \right) \sum_{s}{\left(H\CUMSwapReg(A, N_{h}^{K}\left( s \right), \iota)+\sqrt{H^3N_{h}^{K}(s)\iota}\right)}\\
      \le& \cO\left( HS\CUMSwapReg(A, K/S, \iota)+ \sqrt{H^3SK\iota}\right),
     \end{align*}
  where in the last step we have used concavity.

  Finally take the sum w.r.t. $h\in[H]$ we have
  \begin{equation*}
  \sum_{k=1}^K{\max_{j}[V_{1,j}^{k}-\low{V}_{1,j}^{k}](s_1)} \le \cO\left( H^2S\CUMSwapReg(A, K/S, \iota)+ \sqrt{H^5SK\iota}\right),
  \end{equation*}
  which implies 
  \begin{equation*}
    \max_{j\in[m]}[V_{j,1}^{\dag ,\hat{\pi}_{-j}}(s_1)-V_{j,1}^{\hat{\pi}}(s_1)] \le \cO( (H^2S/K) \cdot \CUMSwapReg(A, K/S, \iota)+\sqrt{H^5S\iota/K}).
    \end{equation*}
 
  \end{proof}

\section{Proofs for MDPs and Two-player Zero-sum MGs}
\label{app:mdp-2p0s}

In this section, we prove the main theorems for V-learning in the setting of single-agent (MDPs) and two-player zero-sum MGs.

\begin{proof}[Proof of Theorem \ref{thm:2p0s_main_result}]	
To begin with, we notice an equivalent  definition of two-player zero-sum MGs is that the reward function satisfies $r_{1,h} = 1- r_{2,h}$ for all $h\in[H]$. The reason we use this definition instead of the common version $r_{1,h}=-r_{2,h}$ is we want to make it consistent with our assumption that the reward function takes value in $[0,1]$ for any player.
Although this definition does not satisfy the zero-sum condition, its Nash equilibria are the same as those of the zero-sum version because adding a constant to the reward function of player $2$ per step will not change the dynamics of the game.

 In order to show $\hat \pi = \hat\pi_1 \times \hat\pi_2$ is an approximate Nash policy, it suffices to control 
 $$\max_{\pi_1}V_{1,1}^{\pi_1,\hat{\pi}_{2}}(s_1)-\min_{\pi_2} V_{1,1}^{\hat{\pi}_1, \pi_2}(s_1).
 $$
 Since $r_{1,h} = 1- r_{2,h}$ for all $h\in[H]$, with probability at least $1-\delta$
 \begin{align*}
 & \max_{\pi_1}V_{1,1}^{\pi_1,\hat{\pi}_{2}}(s_1)-\min_{\pi_2} V_{1,1}^{\hat{\pi}_1, \pi_2}(s_1) \\
 = & \max_{\pi_1}V_{1,1}^{\pi_1,\hat{\pi}_{2}}(s_1)-\left(H-\max_{\pi_2} V_{2,1}^{\hat{\pi}_1, \pi_2}(s_1)\right)\\
 = & \left(\max_{\pi_1}V_{1,1}^{\pi_1,\hat{\pi}_{2}}(s_1) - 	V_{1,1}^{\hat{\pi}_1\odot\hat\pi_2}(s_1)\right)
 + \left(\max_{\pi_2} V_{2,1}^{\hat{\pi}_1, \pi_2}(s_1)- V_{2,1}^{\hat{\pi}_1\odot\hat\pi_2}(s_1)\right)\\
 \le & \cO( (H^2S/K) \cdot \CUMReg(A, K/S, \iota)+\sqrt{H^5S\iota/K}),
 \end{align*}   
 where the last inequality follows from Theorem \ref{thm:CCE_main_result}.
 The reason we can use Theorem \ref{thm:CCE_main_result} here is  the precondition of Theorem \ref{thm:2p0s_main_result} is a special case of  the precondition of Theorem \ref{thm:CCE_main_result}. 
\end{proof}

\begin{proof}[Proof of Theorem \ref{thm:single_main_result}]
	Since MDPs is a subclass of two-player zero-sum MGs by simply choosing the action set of the second player to be a singleton, it suffices to only prove Theorem \ref{thm:2p0s_main_result}, from which the single-agent guarantee, Theorem \ref{thm:single_main_result} trivially follows.  
\end{proof}

\section{Proofs for Monotonic V-learning}

In this section, we prove Theorem \ref{thm:2p0s_monotone}. The algorithm is V-learning with monotonic update, and the setting we consider is two-player zero-sum Markov games. 
As before, we assume $r_{1,h} (s,a)= 1-r_{2,h}(s,a)$ for all $s,a,h$. 
The reason for assuming $r_{1,h} (s,a)= 1-r_{2,h}(s,a)$ instead of 
$r_{1,h} (s,a)= -r_{2,h}(s,a)$ can be found in Appendix \ref{app:mdp-2p0s}.

For two player zero-sum MGs, we can define its minimax value function (Nash value function) by the following Bellman equations 
\begin{equation}
	\begin{cases}
	V^\star_{j,h}(s) = \max_{\pi_{j,h}}\min_{\pi_{-j,h}} \D_{\pi_{j,h}\times \pi_{-j,h}}[Q^\star_{j,h}](s),\\
	  Q^\star_{j,h}(s,\a) = r_{j,h}(s,\a) +\P_h[V^\star_{j,h+1}](s,\a),\\
	 V^\star_{j,H+1}(s)=Q^\star_{j,H+1}(s,\a) = 0.
	\end{cases}
\end{equation}

\begin{lemma}[Optimism of V-estimates]\label{lem:mono_order}
With probability at least $1-\delta$, for any $(s,h,k,j) \in \cS \times [H] \times [K]\times [2]$,  
\begin{equation}
\tilde{V}^k_{j,h}(s)\ge  {V}^k_{j,h}(s)\ge V^{\dagger,\tilde{\pi}_{-j}}_{j,h}(s)\ge V^\star_{j,h}(s),
\end{equation}
where $V^\star_{j,h}$ is the minimax  (Nash) value function defined above.
\end{lemma}
\begin{proof}[Proof of Lemma~\ref{lem:mono_order}]
Note that $\tilde{V}^k_{j,h}(s)\ge  {V}^k_{j,h}(s)$ is straightforward by the update rule of V-learning, and $V^{\dagger,\tilde{\pi}_{-j}}_{j,h}(s)\ge V^\star_{j,h}(s)$  directly follows from the definition of minimax value function. 
Therefore, we only need to prove the second inequality. 
 We do this by backward induction.

       The claim is true for $h=H+1$. Assume for any $s$ and $k$,  ${V}^k_{j,h+1}(s)\ge V^{\dagger,\tilde{\pi}_{-j}}_{j,h+1}(s)$. For a fixed $(s,h,k) \in \cS \times [H]\times[K]$, let $t=N_{h}^{k}\left( s \right) $ and suppose $s$ was previously visited in episode $k^1,\ldots, k^t < k$ at the $h$-th step. By Bellman equation,
   \begin{align*}
      V^{\dagger,\tilde{\pi}_{-j}}_{j,h}(s)
            \le & \alpha _{t}^{0}(H-h+1)+ \underset{\mu}{\max}\sum_{i=1}^t{\alpha _{t}^{i}\D_{\mu \times \pi _{ -j,h}^{k^i}} \left( r_{j,h} +\P_h V_{j,h+1}^{\dagger,\tilde{\pi}_{-j}} \right) \left( s \right)}
      \\
      \le& \alpha _{t}^{0}(H-h+1)+\underset{\mu}{\max}\sum_{i=1}^t{\alpha _{t}^{i}\D_{\mu \times \pi _{ -j,h}^{k^i}} \left( r_{j,h} +\P_h {V}_{j,h+1}^{k^i} \right) \left( s \right)}\\
      \le & \alpha _{t}^{0}(H-h+1)+\sum_{i=1}^t{\alpha _{t}^{i}\D_{\pi_{h}^{k^i}} \left( r_{j,h} +\P_h {V}_{j,h+1}^{k^i} \right)\left( s \right)}+
H \AVGReg(A_j,t,\iota)\\
\le & \alpha _{t}^{0}(H-h+1)+\sum_{i=1}^t{\alpha _{t}^{i}\left[ r_{j,h}(s,\a_h^{k^i})+ {V}_{j,h+1}^{k^i}(s_{h+1}^{k^i})  \right]} + \cO\left(\sqrt{\frac{2H^3\iota}{t}}\right)+H \AVGReg(A_j,t,\iota)
\end{align*}
where the second inequality follows from our induction hypothesis and the monotonicity of $V^k$,  the third inequality follows from Lemma \ref{lem:2p0s-momo-local-regret}, and the last one follows from martingale concentration as well as Lemma \ref{lem:step_size}. By Lemma \ref{lem:2p0s-mono-V-relation} and the precondition of Theorem \ref{thm:2p0s_monotone}, we know the RHS is no larger than  
$\tilde{V}^k_{j,h}(s)$.
Note that $V^k$ can be equivalently defined as
$$
{V}_{j,h}^k(s) =  \min\{ \min_{t\in[k]} \tilde{V}_{j,h}^t(s),H-h+1\}, 
$$
we conclude ${V}^k_{j,h}(s)\ge V^{\dagger,\tilde{\pi}_{-j}}_{j,h}(s)$ for any $k\in[K]$.
\end{proof}


Now we are ready to prove Theorem \ref{thm:2p0s_monotone}.
\begin{proof}[Proof of Theorem \ref{thm:2p0s_monotone}]
By the monotonicity of ${V}$ and Lemma \ref{lem:mono_order}
\begin{align*}
&\quad V^{\dagger,\tilde{\pi}_{2}}_{1,1}(s_1) - \min_{\pi_2}V^{\tilde{\pi}_{1}\times\pi_2}_{1,1}(s_1)\\
&= V^{\dagger,\tilde{\pi}_{2}}_{1,1}(s_1) -\left(H-V^{\dagger,\tilde{\pi}_{1}}_{2,1}(s_1)\right) \\
& \le  {V}^K_{1,1}(s_1)+{V}_{2,1}^K(s_1)-H \\
& \le \frac{1}{K} \sum_{k=1}^{K} \left( {V}^k_{1,1}(s_1)+{V}_{2,1}^k(s_1)-H\right) \\
& \le \frac{1}{K} \sum_{k=1}^{K}\left(\tilde{V}^k_{1,1}(s_1)+\tilde{V}_{2,1}^k(s_1)-H\right),
\end{align*}
where the first equality follows from the definition of two-player zero-sum game, i.e., $r_{1,h}=1-r_{2,h}$.

Now we can mimic the proof of Theorem~\ref{thm:CCE_main_result}.
    Define $\delta _{h}^{k}:= \tilde{V}^k_{1,h}(s_h^k)+\tilde{V}_{2,h}^k(s_h^k)-(H-h+1)$. The non-negativity here follows from Lemma \ref{lem:mono_order} as below
 $$
  \tilde{V}^k_{1,h}(s_h^k)+\tilde{V}_{2,h}^k(s_h^k)-(H-h+1)
  \ge {V}^\star_{1,h}(s_h^k)+{V}_{2,h}^\star(s_h^k)  -(H-h+1)
  = (H-h+1) -(H-h+1) = 0.
 $$
  Let $n_h^k =N_{h}^{k}\left( s_h^k \right) $ and suppose $s_h^k$ was previously visited at episodes $k^1,\ldots, k^{n_h^k} < k$ at the $h$-th step. By Lemma \ref{lem:2p0s-mono-V-relation} and the fact that $r_{1,h} = 1- r_{2,h}$ for all $h$, we have
  \begin{align*}
  \delta _{h}^{k}=&  \tilde{V}^k_{1,h}(s_h^k)+\tilde{V}_{2,h}^k(s_h^k)-(H-h+1)  \\
  = &2\alpha _{n_h^k}^{0}H+\sum_{i=1}^{n_h^k}{\alpha _{n_h^k}^{i}\left[\left( \tilde V_{1,h+1}^{k^i }-\tilde{V}_{2,h+1}^{k^i } \right)\left( s_{h+1}^{k^i} \right)-(H-h)  +\beta_{1,i}+\beta_{2,i} \right]}
  \\
  =&2\alpha _{n_h^k}^{0}H+\sum_{i=1}^{n_h^k}{\alpha _{n_h^k}^{i}\delta _{h+1}^{k^i }}+\cO(H\AVGReg(A, n_h^k, \iota)+\sqrt{H^3\iota/n_h^k})
  \end{align*}
  where in the last step we used $\sum_{i=1}^t\alpha_t^i\beta_{j,i}=\Theta(H\AVGReg(A_j, t, \iota)+\sqrt{H^3\iota/t})$.
  
  The remaining steps follow exactly the same as the proof of Theorem~\ref{thm:CCE_main_result}. As a result, we obtain
\begin{align*}
V^{\dagger,\tilde{\pi}_{2}}_{1,1}(s_1) - \min_{\tilde\pi_2}V^{\tilde{\pi}_{1}\times\tilde\pi_2}_{1,1}(s_1) &\le 
\frac{1}{K} \sum_{k=1}^{K}\left(\tilde{V}^k_{1,1}(s_1)+\tilde{V}_{2,1}^k(s_1)-H\right)\\
&\le \cO\left( \frac{H^2S}{K} \cdot \CUMReg(A, K/S, \iota)+\sqrt{\frac{H^5S\iota}{K}}\right),
\end{align*}
which completes the proof.
\end{proof}


\section{Adversarial Bandit with Weighted External Regret}
\label{sec:bandit}

In this section, we present a Follow-the-Regularized-Leader (FTRL) style algorithm that achieves low weighted (external) regret for the adversarial bandit problem. Although FTRL is a classial algorithm in the adversarial bandit literature, we did not find a good reference of FTRL with changing step size, weighted regret and high probability bound. For completeness of this work, we provide detailed derivations here.



We present the FTRL algorithm in Algorithm \ref{algorithm:FTRL}. In Corollary \ref{cor:single-step-regret}, we prove that FTRL satisfies the Assumption \ref{ass:external_regret} with good regret bounds. Recall that $B$ is the number of actions, and our normalization condition requires loss $\tilde{l}_t \in [0,1]^B$ for any $t$.


\begin{algorithm}[t]
   \caption{FTRL for Weighted External Regret (FTRL)}
   \label{algorithm:FTRL}
   \begin{algorithmic}[1]
      \STATE {\bfseries Initialize:} for any $b \in \cB$, $\theta_{1}(b) \setto 1/B$.
      \FOR{episode $t=1,\dots,K$}
      \STATE Take action $b_t \sim \theta_t(\cdot)$, and observe loss $\tilde{l}_t(b_t)$.
      \STATE $\hat{l}_t(b) \leftarrow \tilde{l}_t( b_t)\mathbb{I}\{b_t=b\}/(\theta_t (b) +\gamma_t)$ for all $b \in \cB$.  
      \label{line:FTRL_loss}
      \STATE $\theta_{t+1}(b) \propto \exp[-(\eta_{t}/w_t) \cdot \sum_{i=1}^{t}w_i \hat{l}_i(b)]$  
      \ENDFOR
   \end{algorithmic}
\end{algorithm}



\begin{corollary}
   \label{cor:single-step-regret}
By choosing hyperparameter $w_t=\alpha _t\left( \prod_{i=2}^t{\left( 1-\alpha _i \right)} \right) ^{-1}$  and $\eta _t=\gamma _t=\sqrt{\frac{H\log B}{Bt}}$, FTRL (Algorithm~\ref{algorithm:FTRL}) satisfies Assumption \ref{ass:external_regret} with
$$
   \AVGReg(B, t, \log(1/\delta)) = 10\sqrt{HB\log(B/\delta) /t},\qquad \CUMReg(B, t, \log(1/\delta)) = 20\sqrt{HBt\log(B/\delta) }
$$
\end{corollary}


To prove Corollary \ref{cor:single-step-regret}, we show a more general weighted regret guarantee which works for any set of weights $\{w_i\}_{i=1}^\infty$ in addition to $\{\alpha_{t}^i\}_{i=1}^t$. In particular, a general weighted regret is defined as 
\begin{equation} \label{eq:weighted_regret}
\AVGRegret(t) =\max_{\theta^{\star}}\sum_{i=1}^t{w_i\left< \theta_i-\theta^{\star}, l_i \right>}
\end{equation}


\begin{theorem}
    \label{thm:adv-bandit}
    For any $t \le K$, following Algorithm~\ref{algorithm:FTRL}, if $\eta _i \le 2 \gamma_i$ and $\eta _i$ is non-increasing for all $i \le t$, let $\iota = \log(B/\delta)$ 
    , then with probability $1-3\delta$, we have
    $$
    \AVGRegret(t) \le \frac{w_t\log B}{\eta _t}+\frac{B}{2}\sum_{i=1}^t{\eta _iw_{i}}+\frac{1}{2}\max _{i\le t}w_i\iota +B\sum_{i=1}^t{\gamma _iw_i}+\sqrt{2\iota \sum_{i=1}^t{w_{i}^{2}}}+\max _{i\le t}w_i\iota /\gamma _t.
    $$
\end{theorem}

We postpone the proof of theorem \ref{thm:adv-bandit} to the end of this section. We first show how to obtain Corollary \ref{cor:single-step-regret} from Theorem \ref{thm:adv-bandit}.

\begin{proof}[Proof of Corollary \ref{cor:single-step-regret}]
The weights $\{w_t\}_{t=1}^K$ we choose satisfy a nice property: for any $t$ we have
$$
\frac{w_i}{w_j}=\frac{\alpha _{t}^{i}}{\alpha _{t}^{j}}.
$$

We prove this for $i\le j$ and the other case is similar. By definition, 
$$
\frac{w_i}{w_j}=\frac{\alpha _i}{\alpha _j}\prod_{k=i+1}^j{\left( 1-\alpha _k \right)},
$$
and 
$$\frac{\alpha _{t}^{i}}{\alpha _{t}^{j}}=\frac{\alpha _i}{\alpha _j}\prod_{k=i+1}^j{\left( 1-\alpha _k \right)}.
$$
We can easily verify that the RHS are the same.

Define $\wAVGRegret(t):=\max_{\theta \in \Delta_\cB}\sum_{i=1}^t \alpha_t^{i} [\la \theta_i, \ell_i \ra - \la \theta, \ell_i\ra ]$. By plugging $w_t=\alpha _t\left( \prod_{i=2}^t{\left( 1-\alpha _i \right)} \right) ^{-1}$ into Theorem~\ref{thm:adv-bandit}, 
and using the property above, we have the regret guarantee

\begin{align*}
   \wAVGRegret(t) \le  \frac{\alpha _t\log B}{\eta _t}+\frac{B}{2}\sum_{i=1}^t{\eta _i\alpha_{t}^i}+\frac{1}{2}\alpha _t\iota +B\sum_{i=1}^t{\gamma _i\alpha_{t}^i}+\sqrt{2\iota \sum_{i=1}^t{\left(\alpha_{t}^i \right) ^2}}+\alpha _t\iota /\gamma _t.
\end{align*}

   By choosing $\eta _t=\gamma _t=\sqrt{\frac{H\log B}{Bt}}$ and using Lemma \ref{lem:step_size}, we can further upper bound the regret by 
   \begin{align*}
      \wAVGRegret(t) \le &\frac{\left( H+1 \right) \log B}{H+t}\sqrt{\frac{Bt}{H\log B}}+\frac{3}{2}\sqrt{HB\log B}\sum_{i=1}^t{\frac{\alpha _{t}^{i}}{\sqrt{t}}}\\
      +&\frac{\left( H+1 \right) \iota}{2\left( H+t \right)}+\sqrt{2\iota \sum_{i=1}^t{\left( \alpha _{t}^{i} \right) ^2}}+\frac{\left( H+1 \right) \iota}{\left( H+t \right)}\sqrt{\frac{Bt}{H\log B}} \\
\le& 2\sqrt{\frac{HB\log B}{t}}+3\sqrt{\frac{HB\log B}{t}}+\frac{H\iota}{t}+2\sqrt{\frac{H\iota}{t}}+2\sqrt{\frac{HB\log B}{t}}\\
\le& 9\sqrt{HB\iota /t}+H\iota /t.
   \end{align*}
   
To further simplify the above upper bound, consider two cases:
\begin{itemize}
   \item If $H\iota /t \le 1$, $\sqrt{H\iota /t} \ge H\iota /t$ and thus $\wAVGRegret(t) \le 10\sqrt{HB\iota /t}$.
   \item If $H\iota /t \ge 1$, $\sqrt{HB\iota /t} \ge 1 \ge \wAVGRegret(t)$ where the last step is by the definition of $\wAVGRegret(t)$. Therefore we have $\wAVGRegret(t) \le \sqrt{HB\iota /t}$.
\end{itemize}
Combining the two cases above gives $\wAVGRegret(t) \le 10\sqrt{HB\iota /t}$.
   
Finally, we pick $\AVGReg(B, t, \log(1/\delta)) := 10\sqrt{HB\iota /t}$, which is non-decreasing in $B$. Since $\sum_{t'=1}^t\AVGReg(B, t, \log(1/\delta)) \le 20\sqrt{HBt\iota }$, we choose $\CUMReg(B, t, \log(1/\delta)) = 20\sqrt{HBt\iota }$, which is concave in $t$.
\end{proof}


To prove Theorem \ref{thm:adv-bandit}, we first note that the weighted regret \eqref{eq:weighted_regret} can be decomposed into three terms
\begin{align}
   \sum_{i=1}^t{w_i\left< \theta_i-\theta^{\star}, l_i \right>} =&\sum_{i=1}^t{w_i\left< \theta_i-\theta^{\star},l_i \right>}
\nonumber \\ 
=&\underset{\left( A \right)}{\underbrace{\sum_{i=1}^t{w_i\left< \theta_i-\theta^{\star},\hat{l}_i \right>}}}+\underset{\left( B \right)}{\underbrace{\sum_{i=1}^t{w_i\left< \theta_i,l_i-\hat{l}_i \right>}}}+\underset{\left( C \right)}{\underbrace{\sum_{i=1}^t{w_i\left< \theta^{\star},\hat{l}_i-l_i \right>}}} \label{eq:regret_decomposition}
\end{align}

The rest of this section is devoted to bounding three terms above. We begin with the following useful lemma adapted from Lemma 1 in \cite{neu2015explore}, which is crucial in achieving high probability guarantees.

\begin{lemma}
   \label{lem:Neu}
   For any sequence of coefficients $c_1, c_2, \ldots, c_t$ s.t. $c_i \in [0,2\gamma_i]^B$ is $\cF_i$-measurable, we have with probability $1-\delta$,
$$
\sum_{i=1}^t{w_i\left< c_i,\hat{l}_i-l_i \right>}\le \max _{i\le t}w_i\iota.
$$
\end{lemma}

\begin{proof}
   Define $w=\max _{i\le t}w_i$. By definition, 
   \begin{align*}
      w_i\hat{l}_i\left( b \right) =&\frac{w_i\tilde{l}_i\left( b \right) \mathbb{I}\left\{ b_i=b \right\}}{\theta_i\left( b \right) +\gamma _i}
      \le \frac{w_i\tilde{l}_i\left( b \right) \mathbb{I}\left\{ b_i=b \right\}}{\theta_i\left( b \right) +\frac{w_i\tilde{l}_i\left( b \right) \mathbb{I}\left\{ b_i=b \right\}}{w}\gamma _i}
      \\
      =&\frac{w}{2\gamma _i}\frac{\frac{2\gamma _iw_i\tilde{l}_i\left( b \right) \mathbb{I}\left\{ b_i=b \right\}}{w \theta_i\left( b \right)}}{1+\frac{\gamma _iw_i\tilde{l}_i\left( b \right) \mathbb{I}\left\{ b_i=b \right\}}{w\theta_i\left( b \right)}}
      \overset{\left( i \right)}{\le} \frac{w}{2\gamma _i}\log \left( 1+\frac{2\gamma _iw_i\tilde{l}_i\left( b \right) \mathbb{I}\left\{ b_i=b \right\}}{w \theta_i\left( b \right)} \right)    
   \end{align*}
   where $(i)$ follows from $\frac{z}{1+z/2}\le \log \left( 1+z \right) $ for all $z \ge 0$.

   Defining the sum
   $$
   \hat{S}_i=\frac{w_i}{w}\left< c_i,\hat{l}_i \right> , \,\,\, S_i=\frac{w_i}{w}\left< c_i,l_i \right>, 
   $$
we have
\begin{align*}
   \E_i\left[ \exp \left( \hat{S}_i \right) \right] &\le \E_i\left[ \exp \left( \sum_b{\frac{c_i\left( b \right)}{2\gamma _i}\log \left( 1+\frac{2\gamma _iw_i\tilde{l}_i\left( b \right) \mathbb{I}\left\{ b_i=b \right\}}{w \theta_i\left( b \right)} \right)} \right) \right] 
\\
&\overset{\left( i \right)}{\le}\E_i\left[ \prod_b{\left( 1+\frac{c_i\left( b \right) w_i\tilde{l}_i\left( b \right) \mathbb{I}\left\{ b_i=b \right\}}{w \theta_i\left( b \right)} \right)} \right] 
\\
&=\E_i\left[ 1+\sum_b{\frac{c_i\left( b \right) w_i\tilde{l}_i\left( b \right) \mathbb{I}\left\{ b_i=b \right\}}{w \theta_i\left( b \right)}} \right] 
\\
&=1+S_i
\le \exp \left( S_i \right) 
\end{align*}
where $(i)$ follows from $z_1\log \left( 1+z_2 \right) \le \log \left( 1+z_1z_2 \right) $ for any $0 \le z_1 \ge 1$ and $z_2 \ge -1$. Note that here we are using the condition $c_i\left( b \right) \le 2\gamma _i$ for all $b\in[B]$.

Equipped with the above bound, we can now prove the concentration result.
\begin{align*}
   \P\left[ \sum_{i=1}^t{\left( \hat{S}_i-S_i \right)}\ge \iota \right] &=\P\left[ \exp \left[ \sum_{i=1}^t{\left( \hat{S}_i-S_i \right)} \right] \ge \frac{B}{\delta} \right] 
\\
&\le \frac{\delta}{B}\E_t\left[ \exp \left[ \sum_{i=1}^t{\left( \hat{S}_i-S_i \right)} \right] \right] 
\\
&\le \frac{\delta}{B}\E_{t-1}\left[ \exp \left[ \sum_{i=1}^{t-1}{\left( \hat{S}_i-S_i \right)} \right] E_t\left[ \exp \left( \hat{S}_t-S_t \right) \right] \right] 
\\
&\le \frac{\delta}{B}\E_{t-1}\left[ \exp \left[ \sum_{i=1}^{t-1}{\left( \hat{S}_i-S_i \right)} \right] \right] 
\\
&\le \cdots \le \frac{\delta}{B}.
\end{align*}
We conclude the proof by taking a union bound.
\end{proof}

With Lemma~\ref{lem:Neu}, we can bound the three terms $(A)$,$(B)$ and $(C)$ in \eqref{eq:regret_decomposition} separately as below.

\begin{lemma}
   \label{lem:bound-A}
   For any $t \in [K]$, suppose $\eta _i \le 2 \gamma_i$  for all $i \le t$. Then with probability at least $1-\delta$,  for any $\theta^{\star} \in \Delta^B$, 
   $$
   \sum_{i=1}^t{w_i\left< \theta_i-\theta^{\star},\hat{l}_i \right>} \le \frac{w_t\log B}{\eta _t}+\frac{B}{2}\sum_{i=1}^t{\eta _iw_{i}}+\frac{1}{2}\max _{i\le t}w_i\iota.
   $$
\end{lemma}
\begin{proof}
   We use the standard analysis of FTRL with changing step size, see for example Exercise 28.13 in \cite{lattimore2018bandit}. Notice the essential step size is $\eta _t/w_t$,
   \begin{align*}
      \sum_{i=1}^t{w_i\left< \theta_i-\theta^{\star},\hat{l}_i \right>}&\le \frac{w_t\log B}{\eta _t}+\frac{1}{2}\sum_{i=1}^t{\eta _i w_i}\left< \theta_i,\hat{l}_{i}^{2} \right> 
\\
&\le \frac{w_t\log B}{\eta _t}+\frac{1}{2}\sum_{i=1}^t{\sum_{b \in \cB}{\eta _i}w_i\hat{l}_i\left( b \right)}
\\
&\overset{\left( i \right)}{\le}\frac{w_t\log B}{\eta _t}+\frac{1}{2}\sum_{i=1}^t{\sum_{b \in \cB}{\eta _i}w_il_i\left( b \right)}+\frac{1}{2}\max _{i\le t}w_i\iota  
\\
&\le \frac{w_t\log B}{\eta _t}+\frac{B}{2}\sum_{i=1}^t{\eta _i w_{i}}+\frac{1}{2}\max _{i\le t}w_i\iota 
   \end{align*}
where $(i)$ is by using Lemma~\ref{lem:Neu} with $c_i(b)=\eta _i$ for any $b$. The any-time guarantee is justifed by taking union bound.
\end{proof}

\begin{lemma}
   \label{lem:bound-B}
   For any $t \in [K]$, with probability $1-\delta$,  
$$
\sum_{i=1}^t{w_i\left< \theta_i,l_i-\hat{l}_i \right>} \le B\sum_{i=1}^t{\gamma _iw_i}+\sqrt{2\iota \sum_{i=1}^t{w_{i}^{2}}}.
$$
\end{lemma}
\begin{proof}
 We further decopose it into 
 $$
 \sum_{i=1}^t{w_i\left< \theta_i,l_i-\hat{l}_i \right>}=\sum_{i=1}^t{w_i\left< \theta_i,l_i-\E_i\hat{l}_i \right>}+\sum_{i=1}^t{w_i\left< \theta_i,\E_i\hat{l}_i-\hat{l}_i \right>}.
$$

The first term is bounded by
\begin{align*}
\sum_{i=1}^t{w_i\left< \theta_i,l_i-\E_i\hat{l}_i \right>}=&\sum_{i=1}^t{w_i\left< \theta_i,l_i-\frac{\theta_i}{\theta_i+\gamma _i}l_i \right>}
\\
=&\sum_{i=1}^t{w_i\left< \theta_i,\frac{\gamma _i}{\theta_i+\gamma _i}l_i \right>}
\le B\sum_{i=1}^t{\gamma _iw_i}.
\end{align*}

To bound the second term, notice
$$
\left< \theta_i,\hat{l}_i \right> \le \sum_{b \in \cB}{\theta_i\left( b \right)}\frac{\mathbb{I}\left\{ b_t=b \right\}}{\theta_i(b)+\gamma _i}\le \sum_{b \in \cB}{\mathbb{I}\left\{ b_i=b \right\}}=1,
$$
thus $\{w_i\left< \theta_i,\E_i\hat{l}_i-\hat{l}_i \right>\}_{i=1}^t$ is a bounded martingale difference sequence w.r.t. the filtration $\{\cF_i\}_{i=1}^t$. By Azuma-Hoeffding,
$$
\sum_{i=1}^t{\left< \theta_i,\E_i\hat{l}_i-\hat{l}_i \right>}\le \sqrt{2\iota \sum_{i=1}^t{w_{i}^{2}}}.
$$
\end{proof}

\begin{lemma}
   \label{lem:bound-C}
   For any $t \in [K]$, with probability $1-\delta$,  for any $\theta^{\star} \in \Delta^B$, if $\gamma_i$ is non-increasing in $i$,
   $$
   \sum_{i=1}^t{w_i\left< \theta^{\star},\hat{l}_i-l_i \right>}\le \max _{i\le t}w_i\iota /\gamma _t.
   $$  
\end{lemma}
\begin{proof}
   Define a basis $\left\{ e_j \right\} _{j=1}^{B}$ of $\mathbb{R}^B$ by 
   $$
   e_j\left( b \right) =\begin{cases}
      \text{1 if }a=j\\
      \text{0 otherwise}\\
   \end{cases}
   $$

   Then for all the $j \in [B]$, we can apply Lemma~\ref{lem:Neu} with $c_i=\gamma_t e_j$. Sincee $c_i(b)\le \gamma_t \le \gamma_i$, the condition in Lemma~\ref{lem:Neu} is satisfied. As a result, 
   $$
   \sum_{i=1}^t{w_i\left< e_j,\hat{l}_i-l_i \right>}\le \max _{i\le t}w_i\iota /\gamma _t.
   $$ 
   
   Since any $\theta^{\star}$ is a convex combination of $\left\{ e_j \right\} _{j=1}^{B}$, by taking the union bound over $j \in [B]$, we have
   $$
   \sum_{i=1}^t{w_i\left< \theta^{\star},\hat{l}_i-l_i \right>}\le \max _{i\le t}w_i\iota /\gamma _t.
   $$
\end{proof}

Finally we are ready to prove Theorem \ref{thm:adv-bandit}.

\begin{proof}[Proof of Theorem \ref{thm:adv-bandit}]
Note the conditions in Lemma~\ref{lem:bound-A} and Lemma~\ref{lem:bound-C} are satisfied by assumptions.
Recall the regret decomposition \eqref{eq:regret_decomposition}. By bounding $(A)$ in Lemma~\ref{lem:bound-A}, $(B)$ in Lemma~\ref{lem:bound-B} and $(C)$ in Lemma~\ref{lem:bound-C}, with probability $1-3\delta$, we have that
 \begin{align*}
   \AVGRegret(t) \le& \frac{w_t\log B}{\eta _t}+\frac{B}{2}\sum_{i=1}^t{\eta _iw_{i}}+\frac{1}{2}\max _{i\le t}w_i\iota +B\sum_{i=1}^t{\gamma _iw_i}+\sqrt{2\iota \sum_{i=1}^t{w_{i}^{2}}}+\max _{i\le t}w_i\iota /\gamma _t.\qedhere
 \end{align*}
\end{proof}


\section{Adversarial Bandit with Weighted Swap Regret}
\label{sec:bandit_swap}

In this section, we adapt Follow-the-Regularized-Leader (FTRL) algorithm that achieves low weighted swap regret for the adversarial bandit problem. We follow a similar technique presented in \cite{blum2007external} which adapts external regret algorithms to swap regret algorithms for the unweighted case.



We present the FTRL\_swap algorithm in Algorithm \ref{algorithm:FTRL_swap}. Different from FTRL (Algorithm~\ref{algorithm:FTRL}), FTRL\_swap maintains an additional $B\times B$ matrix $\tilde\theta_t(\cdot|\cdot)$, and uses its eigenvector when taking actions. The matrix will be updated similarly to FTRL, with a subtle difference that the loss estimator here $\hat{\ell}_t(\cdot|b)$ is $\theta_t(b)$ times the loss estimator $\hat{\ell}_t(\cdot)$ in the FTRL algorithm (Line \ref{line:FTRL_loss} in Algorithm \ref{algorithm:FTRL}).

\begin{algorithm}[t]
   \caption{FTRL for Weighted Swap Regret (FTRL\_swap)}
   \label{algorithm:FTRL_swap}
   \begin{algorithmic}[1]
      \STATE {\bfseries Initialize:} for any $b \in \cB$, $\theta_1(b) \setto 1/B$.
      \FOR{episode $t=1,\dots,K$}    
      \STATE Take action $b_t \sim \theta_t(\cdot)$, and observe loss $\tilde{l}_t(b_t)$.
      \FOR{each action $b \in \cB$}
      \STATE $\hat{l}_t(\cdot|b) \leftarrow \theta_{t}(b) \tilde{l}_t( b_t)\mathbb{I}\{b_t=\cdot\}/(\theta_t (\cdot) +\gamma_t)$.
      \STATE $\tilde{\theta}_{t+1}(\cdot|b) \propto \exp[-(\eta_{t}/w_t) \cdot \sum_{i=1}^{t}w_i \hat{l}_i(\cdot|b)]$
      \ENDFOR  
      \STATE Set $\theta_{t+1}$ such that $ \theta_{t+1}(\cdot) =\sum_a{\theta_{t+1}(b)\tilde{\theta}_{t+1}(\cdot |b)}$. 
      \label{line:stationary_swap}
      \ENDFOR
   \end{algorithmic}
\end{algorithm}

In Corollary \ref{cor:single-step-swap-regret}, we prove that FTRL\_swap satisfies the Assumption \ref{ass:swap_regret} with good swap regret bounds. Recall that $B$ is the number of actions, and our normalization condition requires loss $\tilde{l}_t \in [0,1]^B$ for any $t$.

\begin{corollary}
   \label{cor:single-step-swap-regret}
By choosing hyperparameter $w_t=\alpha _t\left( \prod_{i=2}^t{\left( 1-\alpha _i \right)} \right) ^{-1}$ and $\eta _t=\gamma _t=\sqrt{\frac{H\log B}{t}}$, FTRL\_swap (Algorithm \ref{algorithm:FTRL_swap}) satisfies Assumption \ref{ass:swap_regret} with 
$$
   \AVGSwapReg(B, t, \log(1/\delta)) = 10B\sqrt{H \log(B^2/\delta) /t},\,\,\,\,   \CUMSwapReg(B, t, \log(1/\delta)) = 20B\sqrt{Ht\log(B^2/\delta) }
$$
\end{corollary}

Again, we prove Corollary \ref{cor:single-step-swap-regret} by showing a more general weighted swap regret guarantee which works for any set of weights $\{w_i\}_{i=1}^\infty$ in addition to $\{\alpha_{t}^i\}_{i=1}^t$. A general weighted swap regret is defined as 
\begin{equation} \label{eq:weighted_swap_regret}
\AVGSwapRegret(t):= \min_{\psi \in \Psi}\sum_{i=1}^t w_{i} [\la  \theta_i, l_i \ra- \la \psi\diamond \theta_i, l_i \ra].
\end{equation}

\begin{theorem}
    \label{thm:adv-bandit-swap}
    For any $t \le K$, following Algorithm~\ref{algorithm:FTRL_swap}, if $\eta _i \le 2 \gamma_i$ and $\eta _i$ is non-increasing for all $i \le t$, let $\iota = \log(B^2/\delta)$, then with probability $1-3\delta$, we have
    $$
    \AVGSwapRegret(t) \le  \frac{w_tB\log B}{\eta _t}+\frac{B}{2}\sum_{i=1}^t{\eta _iw_{i}}+\frac{1}{2}\max _{i\le t}w_i\iota + B\sum_{i=1}^t{\gamma _iw_i}+B\sqrt{2\iota \sum_{i=1}^t{w_{i}^{2}}} + B\max _{i\le t}w_i\iota /\gamma _t
    $$
\end{theorem}


We postpone the proof of Theorem \ref{thm:adv-bandit-swap} to the end of this section. We show first how Theorem \ref{thm:adv-bandit-swap} directly implies Corollary \ref{cor:single-step-swap-regret}.
{}

\begin{proof}[Proof of Corollary \ref{cor:single-step-swap-regret}]
As shown in the proof of Corollary~\ref{cor:single-step-regret}, the weights $\{w_t\}_{t=1}^K$ we choose satisfies a nice property: for any $t$ we have
$$
\frac{w_i}{w_j}=\frac{\alpha _{t}^{i}}{\alpha _{t}^{j}}.
$$
Define $\wAVGSwapRegret(t):=\max_{\psi \in \Psi }\sum_{i=1}^t \alpha_t^{i} [\la \theta_i, l_i \ra-\la \psi \diamond\theta_i, l_i\ra ]$. Plugging our choice of $w_i=\alpha _t\left( \prod_{i=2}^t{\left( 1-\alpha _i \right)} \right) ^{-1}$ into Theorem~\ref{thm:adv-bandit-swap}, we have
   \begin{align*}
      \wAVGSwapRegret(t) \le&  \frac{\alpha _tB\log B}{\eta _t}+\frac{B}{2}\sum_{i=1}^t{\eta _i\alpha _{t}^{i}}+\frac{1}{2}\alpha _t\iota \\
      &+B\sum_{i=1}^t{\gamma _i\alpha _{t}^{i}}+B\sqrt{2\iota \sum_{i=1}^t{\left( \alpha _{t}^{i} \right) ^2}}+B\alpha _t\iota /\gamma _t.
   \end{align*}

By choosing $\eta _t=\gamma _t=\sqrt{\frac{H\log B}{t}}$ and using Lemma \ref{lem:step_size},  we can further upper bound the swap regret by 
   \begin{align*}
      \wAVGSwapRegret(t) \le &\frac{\left( H+1 \right) B\log B}{H+t}\sqrt{\frac{t}{H\log B}}+\frac{3B}{2}\sqrt{H\log B}\sum_{i=1}^t{\frac{\alpha _{t}^{i}}{\sqrt{t}}}\\
      +&\frac{\left( H+1 \right) \iota}{2\left( H+t \right)}+B\sqrt{2\iota \sum_{i=1}^t{\left( \alpha _{t}^{i} \right) ^2}}+B\frac{\left( H+1 \right) \iota}{\left( H+t \right)}\sqrt{\frac{t}{H\log B}}\\
\le& 2B\sqrt{H\frac{\log B}{t}}+3B\sqrt{\frac{H\log B}{t}}+\frac{H\iota}{t}+2B\sqrt{\frac{H\iota}{t}}+2B\sqrt{\frac{H\log B}{t}}\\
\le& 9B\sqrt{H\iota /t}+H\iota /t.
   \end{align*}

   To further simplify the above upper bound, consider two cases:
\begin{itemize}
   \item If $H\iota /t \le 1$, $\sqrt{H\iota /t} \ge H\iota /t$ and thus $\wAVGSwapRegret(t) \le 10B\sqrt{H\iota /t}$.
   \item If $H\iota /t \ge 1$, $B\sqrt{H\iota /t} \ge 1 \ge \wAVGSwapRegret(t)$ where the last step is by the definition of $\wAVGSwapRegret(t)$. Therefore we have $\wAVGSwapRegret(t) \le B\sqrt{H\iota /t}$.
\end{itemize}
Combine the above two cases, $\wAVGSwapRegret(t) \le 10B\sqrt{H\iota /t}$.

Finally, we pick $\AVGSwapReg(B, t, \log(1/\delta)) := 10B\sqrt{H\iota /t}$, which is non-decreasing in $B$. On the other hand, since $\sum_{t'=1}^t\AVGSwapReg(B, t, \log(1/\delta)) \le 20B\sqrt{Ht\iota }$, we choose $\CUMSwapReg(B, t, \log(1/\delta)) = 20B\sqrt{Ht\iota }$, which is concave in $t$.
\end{proof}



To prove Theorem \ref{thm:adv-bandit-swap}, we again first decompose the swap regret. We first note that by Line \ref{line:stationary_swap} of Algorithm \ref{algorithm:FTRL_swap}, we have:
\begin{equation*}
w_i  \la \theta_i, l_i\ra =\sum_{b \in \cB}w_i  \la \tilde{\theta}_i(\cdot|b), \theta_i(b)l_i(\cdot)\ra.
\end{equation*}
On the other hand, by the definition of strategy modification $\Psi$, we have
\begin{equation*}
   \min_{\psi \in \Psi}\sum_{i=1}^t w_{i} \la \psi\diamond \theta_i, l_i \ra  = \sum_{b \in \cB} \min_{\theta^{\star}(\cdot|b)} \sum_{i=1}^t w_i \theta_i(b) \cdot \la \theta^{\star}(\cdot|b),l_i(\cdot) \ra.
\end{equation*}
Therefore, we have the following decomposition of the swap regret
\begin{align}
\AVGSwapRegret(t):=& \min_{\psi \in \Psi}\sum_{i=1}^t w_{i} [\la  \theta_i, l_i \ra- \la \psi\diamond \theta_i, l_i \ra]
=\sum_{b \in \cB} \sum_{i=1}^t{w_i [ \la \tilde{\theta}_i(\cdot|b)-\theta^{\star}(\cdot|b),    \theta_i(b)  l_i \ra]} \nonumber \\
=&\underset{\left( A \right)}{\underbrace{\sum_{b \in \cB}\sum_{i=1}^t{w_i\left< \tilde{\theta}_i(\cdot|b)-\theta^{\star}(\cdot|b),\hat{l}_i(\cdot|b) \right>}}}+\underset{\left( B \right)}{\underbrace{\sum_{b \in \cB}\sum_{i=1}^t{w_i\left< \tilde{\theta}_i(\cdot|b),\theta_i(b)l_i(\cdot)-\hat{l}_i(\cdot|b) \right>}}} \nonumber \\
+&\underset{\left( C \right)}{\underbrace{\sum_{b \in \cB}\sum_{i=1}^t{w_i\left< \theta^{\star}(\cdot|b),\hat{l}_i(\cdot|b)-\theta_i(b)l_i(\cdot) \right>}}} \label{eq:swapregret_decompose}
\end{align}
For the remaining proof, we bound term $(A), (B), (C)$ separately in
Lemma~\ref{lem:bound-A-swap}, Lemma~\ref{lem:bound-B-swap}, Lemma~\ref{lem:bound-C-swap}.

\begin{lemma}
   \label{lem:bound-A-swap}
   For any $t \in [K]$, suppose $\eta _i \le 2 \gamma_i$  for all $i \le t$. The with probability $1-\delta$,  for any $\theta^{\star}$, 
   $$
   \sum_{b \in \cB}\sum_{i=1}^t{w_i\left< \tilde{\theta}_i(\cdot|b)-\theta^{\star}(\cdot|b),\hat{l}_i(\cdot|b) \right>}\le \frac{w_tB\log B}{\eta _t}+\frac{B}{2}\sum_{i=1}^t{\eta _iw_{i}}+\frac{1}{2}\max _{i\le t}w_i\iota.
   $$
\end{lemma}
\begin{proof}
   Similar to Lemma~\ref{lem:bound-A}, we have,
   \begin{align*}
      \sum_{i=1}^t{w_i\left< \tilde{\theta}_i(\cdot|b)-\theta^{\star}(\cdot|b),\hat{l}_i(\cdot|b) \right>}&\le \frac{w_t\log B}{\eta _t}+\frac{1}{2}\sum_{i=1}^t{\eta _i w_i}\left< \tilde{\theta}_i(\cdot|b),\hat{l}_{i}^{2}(\cdot|b) \right> 
\\
&= \frac{w_t\log B}{\eta _t}+\frac{1}{2}\sum_{i=1}^t\sum_{b' \in \cB}{\eta _i}w_i\tilde\theta_i(b'|b)\frac{\theta_{i}^2(b) \tilde{l}^2_i( b_i)\mathbb{I}\{b_i=b'\}}{(\theta_i (b') +\gamma_i)^2}
\\
& \le \frac{w_t\log B}{\eta _t}+\frac{1}{2}\sum_{i=1}^t\sum_{b' \in \cB}{\eta _i}w_i\frac{\tilde\theta_i(b'|b)\theta_{i}(b) }{\theta_i (b') }\frac{\hat{l}_{i}(b'|b)}{\theta_i(b)}
\end{align*}

Summing over $b$ and using the fact that $\sum_{b\in \cB} \tilde\theta_i(b'|b)\theta_i(b)=\theta_i(b')$,

\begin{align*}
   \sum_{b \in \cB}\sum_{i=1}^t{w_i\left< \tilde{\theta}_i(\cdot|b)-\theta^{\star}(\cdot|b),\hat{l}_i(\cdot|b) \right>} &\le \frac{w_tB\log B}{\eta _t}+\frac{1}{2}\sum_{i=1}^t{\sum_{b' \in \cB}{\eta _i}w_i\frac{\hat{l}_{i}(b'|b)}{\theta_i(b)}}
\\
&\overset{\left( i \right)}{\le}\frac{w_tB\log B}{\eta _t}+\frac{1}{2}\sum_{i=1}^t{\sum_{b' \in \cB}{\eta _i}w_il_i\left( b' \right)}+\frac{1}{2}\max _{i\le t}w_i\iota  
\\
&\le \frac{w_tB\log B}{\eta _t}+\frac{B}{2}\sum_{i=1}^t{\eta _i w_{i}}+\frac{1}{2}\max _{i\le t}w_i\iota 
   \end{align*}
where $(i)$ is by using Lemma~\ref{lem:Neu} with $c_i(b)=\eta _i$. Notice the quantity 
$\frac{\hat{l}_{i}(b'|b)}{\theta_i(b)}$ actually doesn't depend on $b$, so it is well-defined even after we take the summation with respect to $b$. The any-time guarantee is justified by taking union bound.
\end{proof}

\begin{lemma}
   \label{lem:bound-B-swap}
   For any $t \in [K]$, with probability $1-\delta$ , 
$$
\sum_{b \in \cB}\sum_{i=1}^t{w_i\left< \tilde{\theta}_i(\cdot|b),\theta_i(b)l_i(\cdot)-\hat{l}_i(\cdot|b) \right>} \le B\sum_{i=1}^t{\gamma _iw_i}+B\sqrt{2\iota \sum_{i=1}^t{w_{i}^{2}}}.
$$
\end{lemma}
\begin{proof}
 We further decompose it into 
 $$
 \sum_{i=1}^t{w_i\left< \tilde{\theta}_i(\cdot|b),\theta_i(b)l_i(\cdot)-\hat{l}_i(\cdot|b) \right>}=\sum_{i=1}^t{w_i\left< \tilde{\theta}_i(\cdot|b),\theta_i(b)l_i(\cdot)-\E_i\hat{l}_i(\cdot|b) \right>}+\sum_{i=1}^t{w_i\left< \tilde{\theta}_i(\cdot|b),\E_i\hat{l}_i(\cdot|b)-\hat{l}_i(\cdot|b) \right>}.
$$

The first term is bounded by
\begin{align*}
   \sum_{i=1}^t{w_i\left< \tilde{\theta}_i(\cdot|b),\theta_i(b)l_i(\cdot)-\E_i\hat{l}_i(\cdot|b) \right>}=&\sum_{i=1}^t{w_i\theta_i(b)\left< \tilde{\theta}_i(\cdot|b),(1-\frac{\theta_i(\cdot)}{\theta_i(\cdot)+\gamma _i} )l_i(\cdot)\right>}
\\
=&\sum_{i=1}^t{w_i\theta_i(b)\left< \tilde{\theta}_i(\cdot|b),\frac{\gamma _i}{\theta_i(\cdot)+\gamma _i}l_i \right>}.
\end{align*}

So by taking the sum with respect to $b$, we have 
\begin{align*}
   \sum_{b \in \cB}\sum_{i=1}^t{w_i\left< \tilde{\theta}_i(\cdot|b),\theta_i(b)l_i(\cdot)-\E_i\hat{l}_i(\cdot|b) \right>} \le &\sum_{b \in \cB}\sum_{i=1}^t{w_i\theta_i(b)\left< \tilde{\theta}_i(\cdot|b),\frac{\gamma _i}{\theta_i(\cdot)+\gamma _i}l_i \right>}\\
\le &\sum_{b' \in \cB}\sum_{i=1}^t{w_i\gamma _il_i(b')} \\
\le &B\sum_{i=1}^t{\gamma _iw_i}. 
\end{align*}

To bound the second term, notice $\tilde{\theta}_i(b'|b)\theta_i\left( b \right) \le \theta_i(b')$ for any $b,b' \in \cB$,
$$
\left< \tilde{\theta}_i(\cdot|b),\hat{l}_i(\cdot|b) \right> \le \sum_{b' \in \cB}{\tilde{\theta}_i(b'|b)\theta_i\left( b \right)}\frac{\mathbb{I}\left\{ b_t=b' \right\}}{\theta_i(b')+\gamma _i}\le \sum_{b' \in \cB}{\mathbb{I}\left\{ b_i=b' \right\}}=1,
$$
thus $\{w_i\left< \tilde{\theta}_i(\cdot|b),\E_i\hat{l}_i(\cdot|b)-\hat{l}_i(\cdot|b) \right>\}_{i=1}^t$ is a bounded martingale difference sequence w.r.t. the filtration $\{\cF_i\}_{i=1}^t$. By Azuma-Hoeffding,
$$
\sum_{i=1}^t{w_i\left< \tilde{\theta}_i(\cdot|b),\E_i\hat{l}_i(\cdot|b)-\hat{l}_i(\cdot|b) \right>}\le \sqrt{2\iota \sum_{i=1}^t{w_{i}^{2}}}.
$$

The proof is completed by taking the summation with respect to $b$ and a union bound.
\end{proof}

\begin{lemma}
   \label{lem:bound-C-swap}
   For any $t \in [K]$, suppose $\gamma_i$ is non-increasing in $i$, then
   with probability $1-\delta$,  and any $\theta^{\star} $,
   $$
   \sum_{b \in \cB}\sum_{i=1}^t{w_i\left< \theta^{\star}(\cdot|b),\hat{l}_i(\cdot|b)-\theta_i(b)l_i(\cdot) \right>}\le B\max _{i\le t}w_i\iota /\gamma _t.
   $$  
\end{lemma}
\begin{proof}
   The proof follows from  Lemma~\ref{lem:bound-C} and taking the summation with respect to $b$.
\end{proof}

Finally, we are ready to prove Theorem \ref{thm:adv-bandit-swap}.

\begin{proof}[Proof of Theorem \ref{thm:adv-bandit-swap}]
 

Recall the decomposition of swap regret \eqref{eq:swapregret_decompose}. We bound $(A)$ in Lemma~\ref{lem:bound-A-swap}, $(B)$ in Lemma~\ref{lem:bound-B-swap} and $(C)$ in Lemma~\ref{lem:bound-C-swap}. Putting everything together, we have
\begin{equation*}
   \AVGSwapRegret(t) \le  \frac{w_tB\log B}{\eta _t}+\frac{B}{2}\sum_{i=1}^t{\eta _iw_{i}}+\frac{1}{2}\max _{i\le t}w_i\iota + B\sum_{i=1}^t{\gamma _iw_i}+B\sqrt{2\iota \sum_{i=1}^t{w_{i}^{2}}} + B\max _{i\le t}w_i\iota /\gamma _t.
\end{equation*}
\end{proof}


\end{document}